\documentclass{article} 
\usepackage{iclr2023_conference,times}

\usepackage[utf8]{inputenc} 
\usepackage[T1]{fontenc}  
\usepackage{cancel}
\usepackage{booktabs}
\usepackage{wrapfig}

\usepackage{hyperref}       
\usepackage{url}            
\usepackage{booktabs}       
\usepackage{amsfonts,amsmath,subcaption}       
\usepackage{enumitem,bbm, graphicx}
  \setlist{leftmargin=*,noitemsep}
 \usepackage{nicefrac}       

\usepackage{algorithm,algpseudocode}
\usepackage{YK}
\usepackage{cancel}
\usepackage{graphicx}

\newcommand{\klloss}{\mathfrak{L}}
\def\heta{\widehat{\eta}}
\def\hf{{\hat f}}
\def\hP{\widehat{P}}
\def\hp{\widehat{p}}
\def\hQ{\widehat{Q}}
\def\hq{\widehat{q}}
\def\htheta{\widehat{\theta}}
\def\sf{{f^\star}}
\def\tf{{\tilde f}}
\def\htheta{{\hat \theta}}

\def\hL{\hat L}
\def\sL{{L^\star}}

\newcommand{\breeds}{\textsc{Breeds}}
\newcommand{\waterbirds}{\textsc{Waterbirds}}
\newcommand{\strongconvexity}{\mu}
\newcommand{\nsource}{{n_P}}
\newcommand{\ntarget}{{n_Q}}
\newcommand{\weight}{\omega}
\newcommand{\coef}{\xi}
\newcommand{\scoef}{\coef^\star}
\newcommand{\hcoef}{{\hat \coef}}

\title{Understanding new tasks through the lens of training data via exponential tilting}

\author{%
  	Subha Maity \\
	Department of Statistics\\
	University of Michigan\\
	\texttt{smaity@umich.edu}\\
	\And
	Mikhail Yurochkin \\
	IBM Research\\
	MIT-IBM Watson AI lab\\
	\texttt{mikhail.yurochkin@ibm.com}\\
	\And
	Moulinath Banerjee \\
	Department of Statistics\\
	University of Michigan\\
	\texttt{moulib@umich.edu}\\
	\AND
	Yuekai Sun \\
	Department of Statistics\\
	University of Michigan\\
	\texttt{yuekai@umich.edu}
}

\iclrfinalcopy 
\begin{document}

\maketitle

\begin{abstract}
Deploying machine learning models on new tasks is a major challenge due to differences in distributions of the train (source) data and the new (target) data. However, the training data likely captures some of the properties of the new task. We consider the problem of reweighing the training samples to gain insights into the distribution of the target task.
Specifically, we formulate a distribution shift model based on the exponential tilt assumption and learn train data importance weights minimizing the KL divergence between labeled train and unlabeled target datasets. The learned train data weights can then be used for downstream tasks such as target performance evaluation, fine-tuning, and model selection. We demonstrate the efficacy of our method on \waterbirds\ and \breeds\ benchmarks.
\footnote{Codes can be found in \href{https://github.com/smaityumich/exponential-tilting}{https://github.com/smaityumich/exponential-tilting}.}

\end{abstract}

\section{Introduction}

Machine learning models are often deployed in a target domain that differs from the domain in which they were trained and validated in. This leads to the practical challenges of adapting and evaluating the performance of models on a new domain without costly labeling of the dataset of interest. For example, in the Inclusive Images challenge \citep{shankar2017No}, the training data largely consists of images from countries in North America and Western Europe.
If a model trained on this data is presented with images from countries in Africa and Asia, then (i) it is likely to perform poorly, and (ii) its performance in the training (source) domain may not mirror its performance in the target domain. 
However, due to the presence of a small fraction of images from Africa and Asia in the source data, it may be possible to reweigh the source samples to mimic the target domain.

In this paper, we consider the problem of learning a set of importance weights so that the reweighted source samples closely mimic the distribution of the target domain. We pose an exponential tilt model of the distribution shift between the train and the target data and an accompanying method that leverages unlabeled target data to fit the model. Although similar methods are widely used in statistics \cite{rosenbaum1983central} and machine learning \cite{sugiyama2012Density} to train and evaluate models \emph{under covariate shift} (where the decision function/boundary does not change), one of the main benefits of our approach is it allows \emph{concept drift} (where the decision boundary/function are expected to differ) \citep{cai2019Transfer,gama2014survey} between the source and the target domains.
We summarize our contributions below:
\begin{itemize}
\item In Section \ref{sec:exponential-tilt} we develop a model and an accompanying method for learning source importance weights to mimic the distribution of the target domain \emph{without} labeled target samples.
\item In Section \ref{sec:theory} we establish theoretical guarantees on the quality of the weight estimates and their utility in the downstream tasks of fine-tuning and model selection.
\item We demonstrate applications of our method on \waterbirds\ \citep{sagawa2019Distributionally} (Section \ref{sec:waterbird}),  \breeds\ \citep{santurkar2020breeds} (Section \ref{sec:breeds}) and synthetic (Appendix \ref{sec:normal-mixture}) datasets.  
\end{itemize}

\vspace{-0.4cm}
\section{Related work}

\textbf{Out-of-distribution generalization} is essential for safe deployment of ML models. There are two prevalent problem settings: domain generalization and subpopulation shift \citep{koh2020WILDS}. Domain generalization typically assumes access to several datasets during training that are related to the same task, but differ in their domain or environment \citep{blanchard2011Generalizing,muandet2013Domain}. The goal is to learn a predictor that can generalize to unseen related datasets via learning invariant representations \citep{ganin2016Domainadversarial,sun2016Deep}, invariant risk minimization \citep{arjovsky2019Invariant,krueger2021out}, or meta-learning \citep{dou2019Domain}.
Domain generalization is a very challenging problem and recent benchmark studies demonstrate that corresponding methods rarely improve over vanilla empirical risk minimization (ERM) on the source data unless given access to labeled target data for model selection \citep{gulrajani2020Search,koh2020WILDS}.

Subpopulation shift setting assumes that both train and test data consist of the same groups with different group fractions. This setting is typically approached via distributionally robust optimization (DRO) to maximize worst group performance \citep{duchi2016Statistics,sagawa2019Distributionally}, various reweighing strategies \citep{shimodaira2000Improving,byrd2019What,sagawa2020Investigation,idrissi2021simple}. These methods require group annotations which could be expensive to obtain in practice. Several methods were proposed to sidestep this limitation, however they still rely on a validation set with group annotations for model selection to obtain good performance \citep{hashimoto2018Fairness,liu2021Just,zhai2021DORO,creager2021Environment}. Our method is most appropriate for the subpopulation shift setting (see Section \ref{sec:exponential-tilt}), however it differs in that it does not require group annotations, but requires unlabeled target data.

\textbf{Model selection} on out-of-distribution (OOD) data is an important and challenging problem as noted by several authors \citep{gulrajani2020Search,koh2020WILDS,zhai2021DORO,creager2021Environment}.
\citet{xu2022estimation,chen2021mandoline} propose solutions specific to covariate shift based on parametric bootstrap and reweighing; \citet{garg2022leveraging,guillory2021predicting,yu2022predicting} align model confidence and accuracy with a threshold; \citet{jiang2021assessing,chen2021detecting} train several models and use their ensembles or disagreement. Our importance weighting approach is computationally simpler than the latter and is more flexible in comparison to the former, as it allows for concept drift and can be used in downstream tasks beyond model selection as we demonstrate both theoretically and empirically.

\textbf{Domain adaptation} is another closely related problem setting. Domain adaptation (DA) methods require access to labeled source and unlabeled target domains during training and aim to improve target performance via a combination of distribution matching \citep{ganin2016Domainadversarial,sun2016Deep,shen2018wasserstein}, self-training \citep{shu2018DIRTT,kumar2020Understanding}, data augmentation \citep{cai2021Theory,ruan2021optimal}, and other regularizers. DA methods are typically challenging to train and require retraining for every new target domain. On the other hand, our importance weights are easy to learn for a new domain allowing for efficient fine-tuning, similar to test-time adaptation methods \citep{sun2020TestTime,wang2020Tent,zhang2020Adaptive}, which adjust the model based on the target unlabeled samples. Our importance weights can also be used to define additional regularizers to enhance existing DA methods.

\textbf{Importance weighting} has often been used in the domain adaptation literature on label shift \citep{lipton2018Detecting,azizzadenesheli2019Regularized,maity2022minimax} and covariate shift \citep{sugiyama2007covariate,hashemi2018weighted} but the application has been lacking in the area of concept drift models \citep{cai2019Transfer,maity2021linear}, due to the reason that it is generally impossible to estimate the  weights without seeing labeled data from the target. In this paper, we introduce an exponential tilt model which accommodates concept drift while allowing us to estimate the importance weights for the distribution shift. 
\section{The exponential tilt model}
\label{sec:exponential-tilt}
\paragraph{Notation} We consider a $K$-class classification problem. Let $\cX \in \reals^d$ and $\cY\triangleq[K]$ be the space of inputs and set of possible labels, and $P$ and $Q$ be probability distributions on $\cX\times\cY$ for the source and target domains correspondingly.
A (probabilistic) classifier is a map $f:\cX\to\Delta^{K-1}$.
We define $p\{x, Y = k\}$ as the weighted source class conditional density, \ie\ $p\{x, Y = k\} = p\{x\mid Y = k\} \times P\{Y = k\}$,  where $p\{x \mid Y = k\}$ is the density of the source feature distribution in class $k$ and $P\{Y = k\}$ is the class probability in source. We similarly define $q\{x , Y = k\}$ for target. 

We consider the problem of learning importance weights on samples from a source domain so that the weighted source samples mimic the target distribution. We assume that the learner has access to labeled samples $\{(X_{P,i},Y_{P,i})\}_{i=1}^{n_P}$ from the source domain and and unlabeled samples $\{X_{Q,i}\}_{i=1}^{n_Q}$ from the target domain. The learner's goal is to estimate a weight function $\weight(x, y)>0$ such that
\begin{equation}
\Ex\left[\weight(X_P,Y_P)g(X_P,Y_P)\right] \approx \Ex\big[g(X_Q,Y_Q)\big]\text{ for all (reasonable) }g: \cX \times \cY \to \reals.
\label{eq:true-weighted-expectation}
\end{equation}
Ideally, $\weight = \frac{dQ}{dP}$ is the likelihood ratio between the source and target domains (this leads to equality in \eqref{eq:true-weighted-expectation}), but learning this weight function is generally impossible without labeled samples from the target domain \citep{david2010Impossibility}. Thus we must impose additional restrictions on the domains.

\paragraph{The exponential tilt model} We assume that there is a vector of sufficient statistics $T:\cX \to \reals^p$ and the parameters \{$\theta_k \in \reals^p$, $\alpha_k\in \reals\}_{k=1}^K$ such that 
\begin{equation}
\textstyle\log\frac{q\{x,Y=k\}}{p\{x,Y=k\}} = \theta_k^\top T(x) + \alpha_k\text{ for all }k\in[K];
\label{eq:exponential-tilt}
\end{equation}
\ie\ $q\{x,Y=k\}$ is a member of the exponential family with base measure $p\{x,Y=k\}$ and sufficient statistics $T$. 
We call \eqref{eq:exponential-tilt} the \textbf{exponential tilt} model. It implies the importance weights between the source and target samples are
\[
\weight(x,y) = \exp(\theta_y^\top T(x) +  \alpha_y).
\]
\paragraph{Model motivation}
The exponential tilt model is motivated by the rich theory of exponential families in statistics. In machine learning, it was used for learning with noisy labels and for improving worst-group performance when group annotations are available \citep{li2020tilted,li2021tilted}. It is also closely related to several common models in transfer learning and domain adaptation. In particular, it implies there is a linear concept drift between the source and target domains. It also extends the widely used \textbf{covariate shift} \citep{sugiyama2012Machine} and \textbf{label shift} models \citep{alexandari2020EM,lipton2018Detecting,azizzadenesheli2019Regularized,maity2022minimax,garg2020Unified} of distribution shifts. It extends the covariate shift model because the exponential tilt model permits (linear in $T(X)$) \textbf{concept drifts} between the source and target domains; it extends the label shift model because it allows the class conditionals to differ between the source and target domains. It does, however, come with a limitation: implicit in the model is the assumption that there is some amount of overlap between the source and target domains. In the subpopulation shift setting, this assumption is always satisfied, while in domain generalization it may be violated if the new domain drastically differs from the source data (see Appendix \ref{sec:normal-mixture} for a synthetic data example).

\paragraph{Choosing $T$} 
The goal of $T$ is to identify the common subpopulations across domains, such that 
\[(X_P,Y_P)\mid\{T(X_P) = t, Y_P = k\}\overset{d}{\approx}(X_Q,Y_Q)\mid\{T(X_Q) = t, Y_P = k\}.\]
If $T$ segments the source domain into its subpopulations (\ie\ the subpopulations are $\{(x, y)\in\cX\times \cY\mid T(x) = t, y = k\}$ for different values of $t$'s and $k$'s), then it is possible to achieve perfect reweighing of the source domain with the exponential tilt model: the weight of the $\{T(X) = t, Y = k\}$ subpopulation is $\exp(\theta_k^\top t+\alpha_k)$. However, in practice, such a $T$ that perfectly segments the subpopulations may not exist (\eg\ the subpopulations may overlap) or is very hard to learn (\eg\ we don't have prior knowledge of the subpopulations to guide $T$).

If no prior knowledge of the domains is available, we can use a neural network to parameterize $T$ and learn its weights along with the tilt parameters, or simply use a pre-trained feature extractor as $T$, 
which we demonstrate to be sufficiently effective in our empirical studies.
We also study the effects of misspecification of $T$ using a synthetic dataset example in Appendix \ref{sec:normal-mixture}.




\paragraph{Fitting the exponential tilt model} 
We fit the exponential tilt model via distribution matching. This step is based on the observation that under the exponential tilt model \eqref{eq:exponential-tilt}
\begin{equation} \textstyle
q_X\{x\} = \sum_{k=1}^Kp\{x,Y=k\}\exp(\theta_k^\top T(x) + \alpha_k),
\label{eq:distribution-matching}
\end{equation}
where $q_X$ is the (marginal) density of the inputs in the target domain. It is possible to obtain an estimate $\hq_X$ of $q_X$ from the unlabeled samples $\{X_{i, Q}\}_{i=1}^{n}$ and estimates $\hp\{x,Y=k\}$ of the $p\{x,Y=k\}$'s from the labeled samples $\{(X_{i, P},Y_{i, P})\}_{i=1}^{m}$. This suggests we find $\theta_k$'s and $\alpha_k$'s such that  \[ \textstyle
\sum_{k=1}^K\hp\{x,Y=k\}\exp(\theta_k^\top T(x) + \alpha_k) \approx \hq_X\{x\}.
\] 
Note that the $\theta_k$'s and $\alpha_k$'s are dependent because $\widehat{q}_X$ must integrate to one. We enforce this restriction as a constraint in the distribution matching problem:
\begin{equation}
\{(\htheta_k, \hat \alpha_k)\}_{k=1}^K \in \left\{
\begin{aligned}
& \argmin_{\{(\theta_k, \alpha_k)\}_{k=1}^K}D\left(\textstyle\hq_X\{x\}\|\sum_{k=1}^K\hp\{x,Y=k\}\exp(\theta_k^\top T(x) + \alpha_k)\right)\\
& \text{subject to }\textstyle \int_{\cX} \sum_{k=1}^K\hp\{x,Y=k\}\exp(\theta_k^\top T(x) + \alpha_k)dx = 1\,, 
\end{aligned}
\right.
\label{eq:robust-distribution-matching}
\end{equation}
where $D$ is a discrepancy between probability distributions on $\cX$. Although there are many possible choices of $D$, we pick the Kullback-Leibler (KL) divergence in the rest of this paper because it leads to some computational benefits. We reformulate the above optimization for KL-divergence to relax the constraint which we state in the following lemma.  

\begin{lemma}
\label{lemma:KL-distribution-matching}
If $D$ is the Kullback-Leibler (KL) divergence then optima in \eqref{eq:distribution-matching} is achieved at $\{(\htheta_k, \hat\alpha_k)\}_{k = 1}^K$ where 
\[
\begin{aligned}
\textstyle\{(\htheta_k, \hat \alpha_k')\}_{k=1}^K \in  \argmax_{\{(\theta_k, \alpha_k')\}_{k=1}^K}\Ex_{\hQ_X} \left[\textstyle\log\Big\{\sum_{k=1}^K\heta_{P,k}(X)\exp(\theta_k^\top T(X) + \alpha_k'\Big\}\right]\\
 - \textstyle\log \big\{\Ex_{\hP} \big[ \exp(\theta_Y^\top T(X) + \alpha_Y') \big]\big\}
\end{aligned}\] $\heta_{P} = \{\heta_{P, k}\}_{k = 1}^K$ is a probabilistic classifier for $P$ and $\hat \alpha_k = \hat \alpha_k' - \log \big\{\Ex_{\hP} \big[ \exp(\htheta_Y^\top T(X) +\hat  \alpha_Y') \big]\big\}$. 
\end{lemma}

One benefit of minimizing the KL divergence is that the learner does not need to estimate the $p\{x,Y=k\}$'s, a generative model that is difficult to train.  They merely need to train a discriminative model to estimate $\heta_P$ from the (labeled) samples from the source domain. 

We plug the fitted $\htheta_k$'s and $\hat \alpha_k$'s into \eqref{eq:weights-estimate} to obtain Exponential Tilt Reweighting Alignment (ExTRA) importance weights:
\begin{equation}
\widehat{\weight}(x,y) = \exp(\htheta_y^\top T(x) + \hat \alpha_y).
\label{eq:weights-estimate}
\end{equation}
We summarize the ExTRA procedure in Algorithm \ref{alg:extra} in Appendix \ref{sup:exp:model}.

Next we describe two downstream tasks where ExTRA weights can be used: 
\begin{enumerate}
\item \textbf{ExTRA model evaluation in the target domain.}
Practitioners may estimate the target performance of a model in the target domain by reweighing the empirical risk in the source domain:
\begin{equation}
   \textstyle \label{eq:target-performance}
\Ex\left[\ell(f(X_Q),Y_Q)\right] \approx \frac{1}{n_P}\sum_{i=1}^{n_P}\ell(f(X_{P,i}),Y_{P,i})\widehat{\weight}(X_{P,i}, Y_{P,i}), 
\end{equation}
where $\ell$ is a loss function. This allows to evaluate models in the target domain without target labeled samples \emph{even in the presence of concept drift between the training and target domain}.
\item \textbf{ExTRA fine-tuning for target domain performance.} Since the reweighted empirical risk (in the source domain) is a good estimate of the risk in the target domain, practitioners may fine-tune models for the target domain by minimizing the reweighted empirical risk:
\begin{equation}
\widehat{f}_Q\in\argmin_{f\in\cF}\Ex_{\hP}\left[\ell(f(X),Y)\widehat{\weight}(X, Y)\right].
\label{eq:weighted-ERM}
\end{equation}
\end{enumerate}

We note that the correctness of \eqref{eq:robust-distribution-matching} depends on the identifiability of the $\theta_k$'s and $\alpha_k$'s from  \eqref{eq:distribution-matching}; \ie\ the uniqueness of the parameters that satisfy \eqref{eq:distribution-matching}. As long as the tilt parameters are identifiable, then \eqref{eq:robust-distribution-matching} provides consistent estimates of them. However, without additional assumptions on the $p\{x,Y=k\}$'s and $T$, the tilt parameters are generally unidentifiable from \eqref{eq:distribution-matching}. Next we elaborate on the identifiability of the exponential tilt model.

\section{Theoretical properties of exponential tilting}
\label{sec:theory}
\subsection{Identifiability of the exponential tilt model}
\label{sec:identifiability}
To show that the $\theta_k$'s and $\alpha_k$'s are identifiable from \eqref{eq:distribution-matching}, we must show that there is a unique solution to \eqref{eq:distribution-matching}. Unfortunately, this is not always the case. For example, consider a linear discriminant analysis (LDA) problem in which the class conditionals drift between the source and target domains:
\[
\begin{aligned}
p\{x,Y=k\} &= \pi_k\phi(x-\mu_{P,k}), ~~~
q\{x,Y=k\} &= \pi_k\phi(x-\mu_{Q,k})\,,
\end{aligned}
\]
where $\phi$ is the standard multivariate normal density, $\pi_k\in(0,1)$ are the class proportions in both source and target domains, and $\mu_{P,k}$'s (resp.\ the $\mu_{Q,k}$'s) are the class conditional means in the source (resp.\ target) domains. We see that this problem satisfies the exponential tilt model with $T(x) = x$:
\[ \textstyle
\log\frac{q\{x,Y=k\}}{p\{x,Y=k\}} = (\mu_{Q,k} - \mu_{P,k})^\top x - \frac12\|\mu_{Q,k}\|_2^2 +\frac12\|\mu_{P,k}\|_2^2\,.
\]
This instance of the exponential tilt model is not identifiable. Any permutation of the class labels $\sigma:[K]\to[K]$ also leads to the same (marginal) distribution of inputs:
\[
\begin{aligned}
&\textstyle\sum_{k=1}^Kp\{x,Y=k\}\exp\left(\textstyle(\mu_{Q,k} - \mu_{P,k})^\top x + \frac12\|\mu_{P,k}\|_2^2 - \frac12\|\mu_{Q,k}\|_2^2\right) \\
&\textstyle\quad= \sum_{k=1}^Kp\{x,Y=k\}\exp\left(\textstyle(\mu_{Q,\sigma(k)} - \mu_{P,k})^\top x + \frac12\|\mu_{P,k}\|_2^2 - \frac12\|\mu_{Q,\sigma(k)}\|_2^2\right).
\end{aligned}
\]
From this example, we see that the non-identifiability of the exponential tilt model is closely related to the label switching problem in clustering. Intuitively, the exponential tilt model in the preceding example is too flexible because it can tilt any $p\{x,Y=k\}$ to $q\{x,Y=l\}$. Thus there is ambiguity in which $p\{x,Y=k\}$ tilts to which $q\{x,Y=l\}$. In the rest of this subsection, we present an identification restriction that guarantees the identifiability of the exponential tilt model.

A standard identification restriction in related work on domain adaptation is a clustering assumption. For example, \citet{tachet2020Domain} assume there is a partition of $\cX$ into disjoint sets $\cX_k$ such that $\supp(P\{\cdot\mid Y=k\}),\ \supp(Q\{\cdot\mid Y=k\}) \subset \cX_k$ for all $k\in[K]$. This assumption is strong: it implies there is a perfect classifier in the source and target domains. Here we consider a weaker version of the clustering assumption: there are sets $\cS_k$ such that
\[
P\{Y=k\mid X\in\cS_k\} = Q\{Y=k\mid X\in\cS_k\} = 1.
\]
We note that the $\cS_k$'s can be much smaller than the $\cX_k$'s; this permits the supports of $P\{\cdot\mid Y=k\}$ and $P\{\cdot\mid Y=l\}$ to overlap. 

\begin{definition}[anchor set]
\label{def:anchor-set}
A set $\cS_k\subset\cX$ is an \textbf{anchor set} for class $k$ if $p\{x,Y=k\} > 0$ and $p\{x,Y=l\} = 0$, $l\ne k$ for all $x\in\cS_k$.
\end{definition}


\begin{proposition}[identifiability from anchor sets]
\label{prop:anchor-points} 
If there are anchor sets $\cS_k$ for all $K$ classes (in the source domain) and $T(\cS_k)$ is $p$-dimensional, then there is at most one set of $\theta_k$'s and $\alpha_k$'s that satisfies \eqref{eq:distribution-matching}.
\end{proposition}

This identification restriction is also closely related to the linear independence assumption in \citet{gong2016Domain}. Inspecting the proof of proposition \ref{prop:anchor-points} (see Appendix \ref{supp:prop:anchor-points}), we see that the anchor set assumption implies linear independence of $\{p_k(x)\exp(\theta_k^\top T(x) + \alpha_k)\}_{k=1}^K$ for any set of $\theta_k$'s and $\alpha_k$'s. We study the anchor set assumption empirically in a synthetic experiment in Appendix \ref{sec:normal-mixture}. Our experiments show that the assumption is mild and is violated only under extreme data scenarios.


\subsection{Consistency in estimation of the tilt parameters}

Here, we establish a convergence rate for the estimated tilt parameters (Lemma \eqref{lemma:KL-distribution-matching}) and the ExTRA importance weights (Equation \eqref{lemma:KL-distribution-matching}). To simplify the notation, we define $S(x) = (1, T(x)^\top)^\top$ as the extended sufficient statistics for the exponential tilt and denote the corresponding tilt parameters as $\coef_k = (\alpha_k, \theta_k^\top)^\top$.  We let $\coef_k^\star = (\alpha_k^\star, {\theta_k^\star}^\top )^\top$'s  be the true values of the tilt parameters $\coef_k$'s and let $\coef = (\coef_1^\top, \dots, \coef_K^\top)^\top\in \reals^{K(p+1)}$  be the long vector containing  all the tilt parameters. 
We recall that estimating the parameters from the optimization stated in Lemma  \ref{lemma:KL-distribution-matching}  requires a classifier $\heta_P$ on the source data. So, we define our objective for estimating $\coef$ through a generic  classifier $\eta : \cX \rightarrow \Delta^{K}$. Denoting $\eta_k(x)$ as the $k$-th co-ordinate of $\eta(x)$ we define the expected log-likelihood objective as:
\[
\textstyle
\klloss(\eta, \coef) = \Ex_{Q_X}[\log \{ \sum_{k = 1}^K \eta_k(X) \exp(\coef_k^\top S(X)) \}] - \log [\Ex_{P} \{  \exp(\coef_Y^\top S(X)) \}]\,,
\] and its empirical version as 
\[
\textstyle
\hat \klloss(\eta, \coef) = \Ex_{\hQ_X}[\log \{ \sum_{k = 1}^K \eta_k(X) \exp(\coef_k^\top S(X)) \}] - \log [\Ex_{\hP} \{  \exp(\coef_Y^\top S(X)) \}]\,.
\]



To establish the consistency of MLE we first make an assumption that the loss $\coef \mapsto -\klloss(\eta^\star_P, \coef)$ is strongly convex at the true parameter value. 
\begin{assumption}
\label{assmp:strong-convexity}
The loss $\coef \mapsto -\klloss(\eta^\star_P, \coef)$ is  strongly convex at $\scoef$, \ie, there exists a constant $\strongconvexity>0$ such that for any $\coef$ it holds:
\[ \textstyle
-\klloss(\eta^\star_P, \coef) \ge -\klloss(\eta^\star_P, \scoef) - \partial_\coef\klloss(\eta^\star_P, \scoef)  ^\top (\coef - \scoef) +\frac{\mu}{2}\|\coef - \scoef\|_2^2\,. 
\]
\end{assumption}
We note that the assumption is a restriction on the distribution $Q$ rather than the objective itself. For technical convenience we next assume that the feature space is bounded.

\begin{assumption} \label{assump:bounded-covariate}
$\cX$ is bounded, \ie, there exists  an $M>0$ such that $\cX \subset B(0, M)$.
\end{assumption}

Recall, from Lemma \ref{lemma:KL-distribution-matching}, that we need a fitted source classifier $\heta_P$  to estimate the tilt parameter: $\scoef$ is estimated by maximizing $\hat \klloss(\heta_P, \coef)$ rather than the unknown $\hat \klloss(\eta_P^\star, \coef)$. While analyzing the convergence of $\hcoef$ we are required to control the difference  $\hat \klloss(\hat \eta_P, \coef) - \hat \klloss(\eta_P^\star, \coef)$. 
To ensure the difference is small, assume the pilot estimate of the source regression function $\heta_P$ is consistent at some rate $r_{n_P}$.
\begin{assumption}
\label{assmp:source-classifier}
Let $f_{P, k}^\star(x) = \log\{\eta_{P, k}^\star(x)\} - \frac 1K \sum_{j = 1}^K \log\{\eta_{P, j}^\star(x)\}$. We assume that there exist an estimators $\{\hat f _{P, k}(x)\}_{k =1 }^K$ for $\{f^\star_{P, k}(x)\}_{k = 1}^K$ such that the following holds: there exists a constant $c>0$ and a sequence  $r_{n_P} \to 0$ such that  for almost surely $[\bbP_X]$  it holds
\[\textstyle
\Pr(\|\hat f_P(x) - f_P^\star(x)\|_{2} > t) \le \exp(-c t^2 /r_{n_P}^2), ~ t > 0\,.
\]
\end{assumption}
We use the estimated logits $\{\hat f _{P, k}(x)\}_{k =1 }^K$ to construct the regression functions as  $\heta_{P, k}(x) = \exp(\hat f_{P, k}(x))/\{\sum_{j = 1}^K \exp(\hat f_{P, j}(x))\}$, which we use in the objective stated in Lemma  \ref{lemma:KL-distribution-matching} to analyze the convergence of the tilt parameter estimates and the ExTRA weights.  With the above assumptions we're now ready to state concentration bounds for $\hcoef - \scoef$ and  $\hat \weight - \weight^\star$, where  the true importance weight $\weight^\star$ is  defined as $\weight^\star(x, y) = \exp({\coef_y^\star}^\top S(x))$. 

\begin{theorem} \label{thm:tilt-concentration}
Let the assumptions \ref{assmp:strong-convexity}, \ref{assump:bounded-covariate} and \ref{assmp:source-classifier} hold. For the sample sizes $n_P, n_Q$ define $\alpha_{n_P, n_Q} = r_{n_P} \sqrt{\log(n_Q)}+ {\{(p+1)K/{n_P}\} }^{1/2} + {\{(p+1)K/n_Q\} }^{1/2}$. There exists constants $k_1, k_2>0$ such that for any $\delta>0$ with probability at least $1 - (2K+ 1)\delta$ the following hold:
\[
\|\hcoef - \scoef \|_2 \le k_1 \alpha_{n_P, n_Q}\sqrt{\log(1/\delta)}, \text{ and } \|\hat \weight - \weight^\star\|_{1, P} \le k_2 \alpha_{n_P, n_Q}\sqrt{\log(1/\delta)}. 
\]
\end{theorem}
In Theorem \ref{thm:tilt-concentration} we notice that as long as $r_{n_P} \log(n_Q) \to 0$ for $n_P, n_Q\to \infty$ we have $\alpha_{n_P,n_Q} \to 0$. This implies both \emph{the estimated tilt parameters and the ExTRA weights converge to their true values as the sample sizes $n_P, n_Q\to \infty$}.

We next provide theoretical guarantees for the downstream tasks (1) fine-tuning and (2) target performance evaluation that we described in Section \ref{sec:exponential-tilt}. 

\paragraph{Fine-tuning}
We establish a generalization bound for the fitted model \eqref{eq:weighted-ERM} using weighted-ERM on source domain.
We denote $\cF$ as the classifier hypothesis class. For $f \in \cF$ and a weight function $\weight: \cX \times \cY \to \reals_{\ge 0} $ define the weighted loss function and its empirical version on source data as: 
\[ \textstyle
\cL_P(f, w) = \Ex_{P} [\weight(X, Y)\ell(f(X), Y)], ~ \hat \cL_P (f, w) = \Ex_{\hP} [\weight(X, Y)\ell(f(X), Y)]\,.
\] We also define the loss function on the target data as: $
\cL_Q(f) = \Ex_{Q} \big[\ell(f(X), Y)\big]\,.$
If $\{(\theta_k^\star, \alpha_k^\star)\}_{k = 1}^K$ is the true value of the tilt parameters in  \eqref{eq:exponential-tilt}, \ie, the following holds:
\[ \textstyle
  q\{x, Y = k\} = p\{x, Y = k\} \exp\{\alpha_k^\star + (\theta_k^\star)^\top T(x)\}; \ k \in [K]\,,
\] then defining $\weight^\star(x, k) = \exp\{\alpha_k^\star + (\theta_k^\star)^\top T(x)  \}$ as the true weight we notice that  $\cL_P(f, \weight^\star) = \cL_Q(f)$, which is easily observed by setting $g(x, y) = \ell(f(x), y)$ in the display \eqref{eq:true-weighted-expectation}.



To establish our generalization bound we require Rademacher complexity \citep{bartlett2002rademacher} (denoted as $\cR_{n_P}(\cG)$; see Appendix \ref{supp:rademacher} for details) and the following assumption on the loss function.

\begin{assumption}
\label{assump:bounded-loss}
The loss function $\ell$ is bounded, \ie, for some $B>0$, $|\ell\{f(x), y\}| \le B$ for any $f\in \cF$, $x \in \cX$ and $y \in [K]$. 
\end{assumption}

With the above definitions and the assumption we  establish our generalization bound. 

\begin{lemma}
\label{lemma:generalization-bound} For a weight function $\weight$ and the source samples $\{(X_{P, i}, Y_{P, i})\}_{i = 1}^\nsource$ of size $\nsource$ let $\hat f_{\weight} = \argmin_{f\in\cF} \hat \cL_P(f, \weight)$. There exists a constant $c> 0$ such that  the following generalization bound holds with probability at least $1 -\delta$
\begin{equation}\textstyle
\label{eq:gen-bound}
    \cL_Q(\hat f_{ \weight}) - \min_{f\in\cF}\cL_Q (f) \le  2 \cR_{n_P}(\cG) + B  \| \weight - \weight^\star \|_{1, P} + c \sqrt{\frac{\log(1/\delta)}{\nsource}} \,,
\end{equation} where $\cR_{n_P}(\cG)$ is the Rademacher complexity of $\cG = \{ \weight^\star(x, y) \ell(f(x), y): \ f \in \cF\}$  defined in Appendix A.1. 
\end{lemma}
In Theorem \ref{thm:tilt-concentration} we established an upper bound for the estimated weights $\hat \weight$, which concludes that $\hat f_{\hat \weight}$ has the following generalization bound:  for any $\delta > 0$, with probability at least $1 - (2K+ 2)\delta$
\[ \textstyle
\cL_Q(\hat f_{ \hat \weight}) - \min_{f\in\cF}\cL_Q (f) \le  2 \cR_{n_P}(\cG)  + k_2\alpha_{\nsource, \ntarget} \sqrt{\log(1/\delta)} + c \sqrt{\log(1/\delta)/\nsource} \,,
\] where $k_2$ is the constant in Theorem \ref{thm:tilt-concentration} and $c$ is the constant in Lemma \ref{lemma:generalization-bound}. The generalization bound implies that for large sample sizes ($n_P, n_Q\to \infty$) the target accuracy of weighted ERM on source data well approximates the accuracy of ERM on target data.

\paragraph{Target performance evaluation}

We provide a theoretical guarantee for the target performance evaluation \eqref{eq:target-performance} using our importance weights. Here we only consider the functions $g:\cX \times \cY \to \reals$ which are bounded by some $B>0$, \ie\ $|g(x, y)| \le B$ for all $x\in \cX$ and $y \in \cY$. The simplest and the most frequently used example is the model accuracy which uses $0\text{-} 1$-loss as the loss function: for a model $f$ the loss $g(x, y) = \bbI\{f(x) = y\}$ is bounded with $B = 1$. For such functions  we notice that $\Ex_{Q}[g(X, Y)] = \Ex_{P}[g(X, Y)\weight^\star(X, Y)]$, as observed in display \eqref{eq:true-weighted-expectation}.
This implies the following bound on the target performance evaluation error 
\[
\begin{aligned}
\big|\Ex_{Q}[g(X, Y)] - \Ex_{P}[g(X, Y)\hat \weight(X, Y)]\big|& = \big|\Ex_{P}[g(X, Y)\weight^\star(X, Y)] - \Ex_{P}[g(X, Y)\hat \weight(X, Y)]\big|\\
& \le B \Ex_{P}[|\hat \weight^\star(X, Y)- \weight^\star (X, Y)|]
 \le B \|\hat \weight - \weight^\star\|_{1, P}\,.
\end{aligned}
\] We recall the concentration bound for $\|\hat \weight - \weight^\star\|_{1, P}$ from Theorem \ref{thm:tilt-concentration} and conclude that the \emph{estimated target performance in \eqref{eq:target-performance} converges to the true target performance} at rate $\alpha_{\nsource, \ntarget}$. 
\section{\waterbirds\ case study}
\label{sec:waterbird}
To demonstrate the efficacy of the ExTRA algorithm for reweighing the source data we (i) verify the ability of ExTRA to upweigh samples most relevant to the target task; (ii) evaluate the utility of weights in downstream tasks such as fine-tuning and (iii) model selection.

\textbf{\waterbirds\ dataset} combines bird photographs from the Caltech-UCSD Birds-200-2011 (CUB) dataset \citep{wah2011caltech} and the image backgrounds from the Places dataset \citep{zhou2017places}. The birds are labeled  as one of
$\cY =$ \{waterbird, landbird\} and placed on one of $\cA =$ \{water background, land background\}.
The images are divided into four groups: landbirds on land (0); landbirds on water (1); waterbirds on land (2); waterbirds on water (3). The source dataset is highly imbalanced, i.e. the smallest group (2) has 56 samples. We embed all images with a pre-trained ResNet18 \citep{he2016Deep}. See Appendix
\ref{sup:exp:data}
for details.

We consider five subpopulation shift target domains: all pairs of domains with different bird types and the original test set \citep{sagawa2019Distributionally} where all 4 groups are present with proportions vastly different from the source.
For all domains, we fit ExTRA weights (with ResNet18 features as $T(x)$) from 10 different initializations and report means and standard deviations for the metrics. See Appendix \ref{sup:exp:model}
for the implementation details.

\textbf{ExTRA weights quality} For a given target domain it is most valuable to upweigh the samples in the source data corresponding to the groups comprising that domain. The most challenging is the target $\{1,2\}$ consisting only of birds appearing on their atypical backgrounds. Groups $\{1,2\}$ correspond to 5\% of the source data making them most difficult to ``find''. To quantify the ability of ExTRA to upweigh these samples we report precision (proportion of samples from groups $\{1,2\}$ within the top $x\%$ of the weights) and recall (proportion of $\{1,2\}$ samples within the top $x\%$ of the weights) in Figure \ref{fig:precision-recall}. We notice that samples corresponding to $10\%$ largest ExTRA weights contain slightly over $80\%$ of the groups $\{1,2\}$ in the source data (recall). This demonstrates the ability of ExTRA to upweigh relevant samples. We present examples of upweighted images and results for other target domains in Appendix
\ref{sup:exp:results}.

\begin{wrapfigure}[13]{r}{0.4\linewidth}
\vspace{-0.45cm}
\centering
\includegraphics[width=\linewidth]{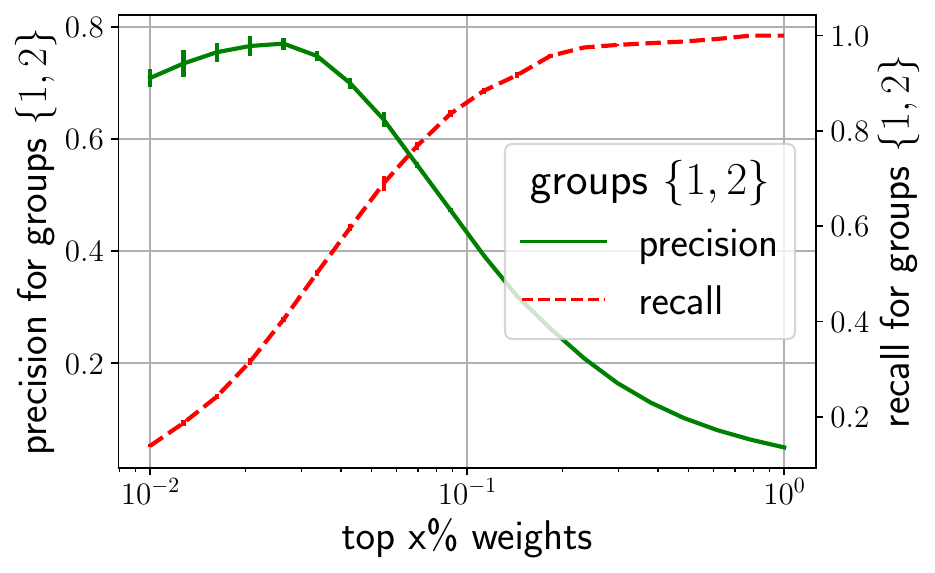}
\vspace{-.7cm}
\caption{ExTRA precision and recall for samples with top $x\%$ weights.}
\label{fig:precision-recall}
\end{wrapfigure}

\textbf{Model fine-tuning}
We demonstrate the utility of ExTRA weights in the fine-tuning  \eqref{eq:weighted-ERM}. The basic goal of such importance weighing is to improve the performance in the target in comparison to training on uniform source weights S -> T, \ie\  ERM. Another baseline is the DRO model \citep{hashimoto2018Fairness} that aims to maximize worst-group performance without access to the group labels, and JTT \citep{liu2021Just} that retrains a model after upweighting the misclassified samples by ERM. We consider two additional baselines that utilize group annotations to improve worst-group performance: re-weighing the source to equalize group proportions (RW$_\text{gr}$) and group DRO (gDRO) \citep{sagawa2019Distributionally}. The aforementioned baselines do not try to adjust to the target domain. Finally, we compare to $\pi$T -> T that fine-tunes the model only using the subset of the source samples corresponding to the target domain groups. In all cases we use logistic regression as model class.

\begin{wrapfigure}[14]{r}{0.45\linewidth}
\vspace{-0.2in}
    \centering
    \includegraphics[width=\linewidth]{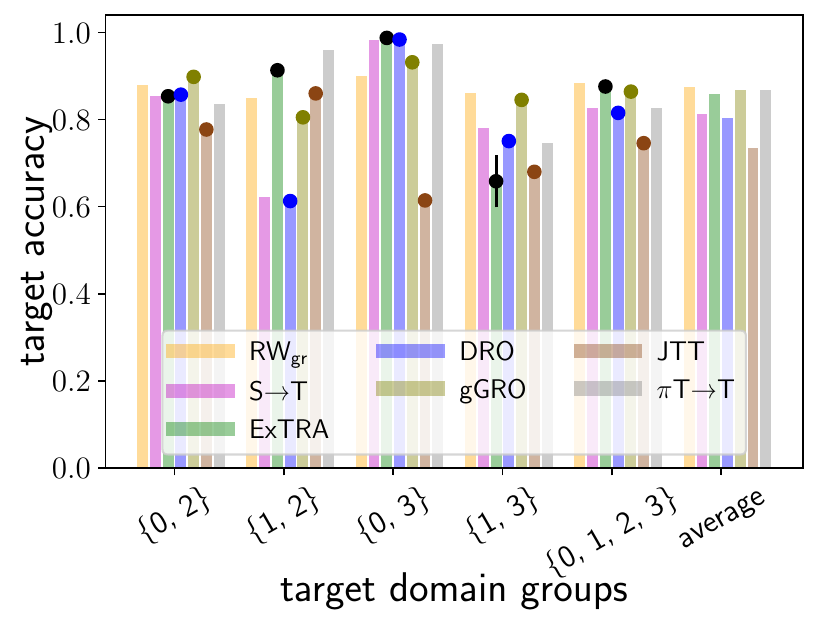}
    \vspace{-0.25in}
    \caption{Performance on \waterbirds.}
    \label{fig:waterbirds}
\end{wrapfigure}

We compare target accuracy across domains in Figure \ref{fig:waterbirds}. Analogous comparison with area under the receiver operator curve can be found in Figure \ref{fig:waterbird-auc} in Appendix \ref{sup:exp:waterbirds}. Model trained with ExTRA weights outperforms all ``fair'' baselines and matches the performance of the three baselines that had access to additional information. In all target domains ExTRA fine-tuning is comparable with the $\pi$T -> T supporting its ability to upweigh relevant samples. Notably, on \{1,2\} domain of both minority groups \emph{and} on \{0,3\} domain of both majority groups, ExTRA outperforms RW$_\text{gr}$ and gDRO that utilize group annotations. This emphasizes the advantage of adapting to the target domain instead of pursuing a more conservative goal of worst-group performance maximization. Finally, we note that ExTRA fine-tuning did not perform as well on the domain \{1,3\}, however neither did $\pi$T -> T.

\begin{table}[]
\caption{Model selection results on \waterbirds}
    \centering

{\small
\begin{tabular}{l@{\hskip 0.3in}ccc@{\hskip 0.3in}ccc}
\toprule
{} & \multicolumn{3}{c}{target accuracy} & \multicolumn{3}{c}{rank correlation} \\
\cmidrule[1pt](lr){2-7}
target groups &            ExTRA & SrcVal & ATC-NE &            ExTRA & SrcVal & ATC-NE \\
\midrule
\{0, 2\}       &  0.819$\pm$0.012 &  0.854 &  \textbf{0.871} &   0.419$\pm$0.01 &  \textbf{0.807} &  0.760 \\
\{1, 2\}       &  \textbf{0.741}$\pm$0.047 &  0.616 &  0.646 &  \textbf{0.747}$\pm$0.106 & -0.519 & -0.590 \\
\{0, 3\}       &  \textbf{0.978}$\pm$0.001 &  \textbf{0.978} &  0.976 &  \textbf{0.962}$\pm$0.004 &  0.956 &  0.906 \\
\{1, 3\}       &  \textbf{0.757}$\pm$0.011 &  0.737 &  0.747 &  \textbf{0.361}$\pm$0.168 & -0.318 & -0.411 \\
\{0, 1, 2, 3\} &  \textbf{0.856}$\pm$0.034 &  0.803 &  0.818 &  \textbf{0.658}$\pm$0.295 &  0.263 &  0.178 \\
\midrule
average      &     \textbf{0.83} &  0.798 &  0.812 &    \textbf{0.753} &  0.166 &  0.110 \\
\bottomrule
\end{tabular}}
   \vspace{-0.4cm}
    \label{tab:waterbirds-model-validation}
\end{table}

\textbf{Model selection} out-of-distribution is an important task, that is difficult to perform without target data labels and group annotations \citep{gulrajani2020Search,zhai2021DORO}. We evaluate the ability of choosing a model for the target domain based on accuracy on the ExTRA reweighted source validation data. We compare to the standard source validation model selection (SrcVal) and to the recently proposed ATC-NE \citep{garg2022leveraging} that uses negative entropy of the predicted probabilities on the target domain to score models. We fit a total of 120 logistic regression models with different weighting (uniform, label balancing, and group balancing) and varying regularizers. See Appendix
\ref{sup:exp:model}
for details.

In Table \ref{tab:waterbirds-model-validation} we compare the target performance of models selected using each of the model evaluation scores and rank correlation between the corresponding model scores and true target accuracies. Model selection with ExTRA results in the best target performance and rank correlation on 4 out of 5 domains and on average. Importantly, the rank correlation between the true performance and ExTRA model scores is always positive, unlike the baselines, suggesting its reliability in providing meaningful information about the target domain performance.

\vspace{-0.3cm}
\section{\breeds\ case study}
\label{sec:breeds}
\breeds\ \citep{santurkar2020breeds} is a subpopulation shift benchmark derived from ImageNet \citep{deng2009imagenet}. It uses the class hierarchy to define groups within classes. For example, in the Entity-30 task considered in this experiment, class fruit is represented by strawberry, pineapple, jackfruit, Granny Smith in the source and buckeye, corn, ear, acorn in the target. This is an extreme case of subpopulation shift where source and target groups have zero overlap. We modify the dataset by adding a small fraction $\pi$ of random samples from the target to the source for two reasons: (i) our exponential tilt model requires some amount of overlap between source and target; (ii) arguably, in practice, it is more likely that the source dataset has at least a small representation of all groups.

\begin{wrapfigure}[12]{r}{0.4\linewidth}
    \vspace{-0.4cm}
    \centering
    \includegraphics[width=\linewidth]{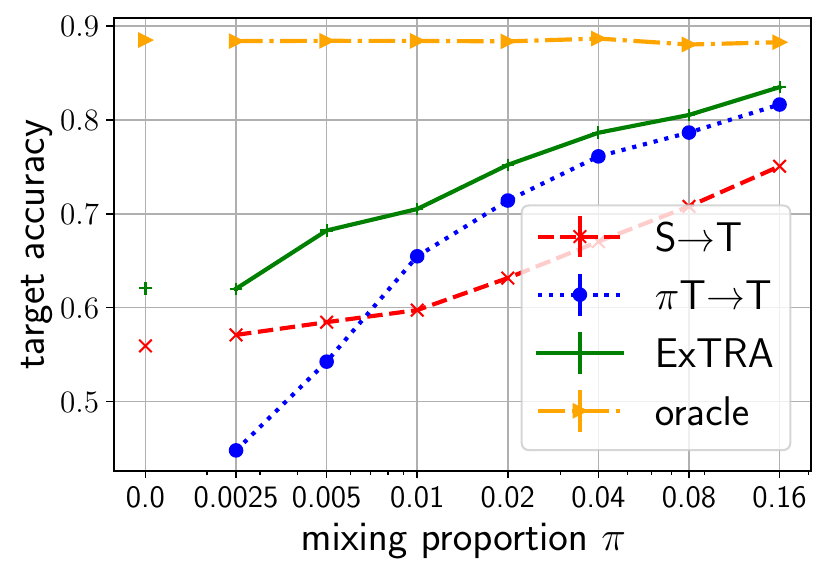}
    \vspace{-0.6cm}
    \caption{Performance on \breeds.}
    \label{fig:breeds}
\end{wrapfigure}

Our goal is to show that ExTRA can identify the target samples mixed into the source for efficient fine-tuning. We obtain feature representations from a pre-trained self-supervised SwAV \citep{caron2020unsupervised}. To obtain the ExTRA weights we use SwAV features as sufficient statistic.  We then train logistic regression models on (i) the source dataset re-weighted with ExTRA, (ii) uniformly weighted source (S -> T), (iii) target samples mixed into the source ($\pi$T -> T), (iv) all target samples (oracle). See Appendix \ref{sup:exp:data}, \ref{sup:exp:model}
for details. We report performance for varying mixing proportion $\pi$ in Figure \ref{fig:breeds}. First, we note that even when $\pi=0$, i.e. source and target have completely disjoint groups (similar to domain generalization), ExTRA improves over the vanilla S -> T. Next, we see that S -> T improves very slowly in comparison to ExTRA as we increase the mixing proportion; $\pi$T -> T improves faster as we increase the number of target samples it has access to, but never suppresses ExTRA and matches its improvement slope for the larger $\pi$ values. We conclude that ExTRA can effectively identify target samples mixed into source that are crucial for the success of fine-tuning \emph{and} find source samples most relevant to the target task allowing it to outperform $\pi$T -> T. We report analogous precision and recall for the \waterbirds\ experiment in Appendix
\ref{sup:exp:results}.

\vspace{-0.25cm}
\section{Conclusion}
In this paper, we developed an importance weighing method for approximating expectations of interest on new domains leveraging unlabeled samples (in addition to a labeled dataset from the source domain). We demonstrated the applicability of our method on downstream tasks such as model evaluation/selection and fine-tuning both theoretically and empirically.
Unlike other importance weighing methods that only allow covariate shift between the source and target domains, we permit concept drift between the source and target. Though we demonstrate the efficacy of our method in synthetic setup of concept drift (Appendix \ref{sec:normal-mixture}), in a future research it would be interesting to investigate the performance in more realistic setups (\eg\ 
 CIFAR10.2 to CIFAR10.2 \citep{lu2020harder}, Imagenet to Imagenetv2 \citep{recht2019imagenet}). 

Despite its benefits, the exponential tilt model does suffer from a few limitations. Implicit in the exponential tilt assumption is that the supports of the target class conditionals have some overlap with the corresponding source class conditionals. Although this assumption is likely satisfied 
in many instances of 
domain generalization problems (and is always satisfied in the subpopulation shift setting), 
an interesting avenue for future studies is to accommodate support alignment in the distribution shift model, \ie\ to align the supports for class conditioned feature distributions in source and target domains. One way to approach this is to utilize distribution matching techniques from domain adaptation literature \citep{ganin2016Domainadversarial,sun2016Deep,shen2018wasserstein}, similarly to \citet{cai2021Theory}. We hope aligning supports via distribution matching will allow our method to succeed on domain generalization problems where the support overlap assumption is violated.

\section{Ethics Statement} We recommend  considering the representation of the minority groups when applying ExTRA in the context of fairness-sensitive applications. The goal of ExTRA is to approximate the distribution of the target domain, thus, in order to use ExTRA weights for fine-tuning or model selection to obtain a fair model, the target domain should be well representative of both privileged and unprivileged groups. If the target domain has miss/under-represented groups, a model obtained using ExTRA weights may be biased.

\subsubsection*{Acknowledgments}
This paper is based upon work supported by the National Science Foundation (NSF) under grants no.\ 2027737 and 2113373. 

\bibliography{YK,SM}
\bibliographystyle{iclr2023_conference}
\newpage

\appendix

\section{Proofs}

\subsection{Rademacher complexity}
\label{supp:rademacher}
We next define the Rademacher complexity \citep{bartlett2002rademacher} that  has been frequently used in  machine learning literature to establish a generalization bound. 
Instead of considering the Rademacher complexity on $\cF$ we define the class of weighted losses $\cG(\ell, \cF) = \{g_f(x, y) = w^\star(x, y) \ell(f(x), y): \ f \in \cF\}$  and for $n \in \bbN$ we define its Rademacher complexity measure as 
\[ \textstyle
    \cR_n(\cG) \triangleq \Ex_{\{(u_i, v_i)\}_{i = 1}^n\stackrel{\IID}{\sim} P} \big[\Ex_{\{\xi_i\}_{i = 1}^n}[\sup_{f \in \cF} \frac1{n} \sum_{i = 1}^n \xi_i \weight^\star(u_i, v_i) \ell \{f(u_i), v_i\}]\big]
\]
where $\{\xi_i\}_{i = 1}^n $ are \IID \

\subsection{Proof of lemma \ref{lemma:KL-distribution-matching}}

If $D$ is the Kulback-Leibler (KL) divergence, then we can rewrite the objective in  \eqref{eq:robust-distribution-matching}  as:
\[
\begin{aligned}
& \Ex_{\hQ_X}\left[\textstyle\log\hq_X\{X\} - \log\Big\{\sum_{k=1}^K\hp\{X,Y=k\}\exp(\theta_k^\top T(X) + \alpha_k)\Big\}\right] \\
&\quad=\textstyle\Ex_{\hQ_X}\left[\textstyle\log\Big\{\sum_{k=1}^K\heta_{P,k}(X)\exp(\theta_k^\top T(X) + \alpha_k\Big\}\right] - \Ex_{\hQ_X} \Big[\log\Big\{\frac{q_X\{x\}}{p_X\{x\}}\Big\}\Big],
\end{aligned}
\] where the term $-\Ex_{\hQ_X} \big[\log\big\{\frac{q_X\{x\}}{p_X\{x\}}\big\}\big]$ in our objective does not involve any tilt parameters and we drop it from our objective. To simplify the notations we define \[
\begin{aligned}
O(\theta, \alpha) &\triangleq \Ex_{\hQ_X}\left[\textstyle\log\Big\{\sum_{k=1}^K\heta_{P,k}(X)\exp(\theta_k^\top T(X) + \alpha_k\Big\}\right]\\
N(\theta, \alpha) &\triangleq \Ex_{\hP} \big[ \exp(\theta_Y^\top T(X) + \alpha_Y) \big]\,.
\end{aligned}
\] 
where $(\theta, \alpha) \in \reals^q$ for $q = K(p+1)$. 
In terms of $O$ and $N$, \eqref{eq:distribution-matching} is
\begin{equation} \label{eq:objs1}
    (\hat \theta , \hat \alpha) = \argmax_{(\theta, \alpha)} \left\{ O(\theta, \alpha)\mid  N(\theta, \alpha) = 1 \right\}.
\end{equation}
Let $F_1 \triangleq \{ (\theta, \alpha)\mid  N(\theta, \alpha) = 1 \}$ be the feasible set. We introduce a change of variables:
\[
c(\theta , \alpha') = (\theta, \alpha(\theta, \alpha'))  ~ \text{ where } ~ \alpha(\theta, \alpha') = \alpha' - 1_K\log(N(\theta, \alpha')).
\] Note that for any $(\theta , \alpha') \in \reals^q,~ c(\theta , \alpha') \in F_1$ because  it holds: 
\[
\begin{aligned}
N(c(\theta, \alpha')) & =  N(\theta, \alpha' - \log(N(\theta, \alpha')) \times 1_K)\\
& = \Ex_{\hP} \big[ \exp(\theta_Y^\top T(X) + \alpha_Y' - \log(N(\theta, \alpha'))) \big]\\
& = \frac{\Ex_{\hP} \big[ \exp(\theta_Y^\top T(X) + \alpha_Y' ) \big]}{N(\theta, \alpha')} = 1\,.
\end{aligned}
\] and the objective value changes to 
\[
\begin{aligned}
O(c(\theta, \alpha')) & = \Ex_{\hQ_X}\left[\textstyle\log\Big\{\sum_{k=1}^K\heta_{P,k}(X)\exp[\theta_k^\top T(X) + \alpha_k'  - \log(N(\theta, \alpha')]\Big\}\right] \\
& = \Ex_{\hQ_X}\left[\textstyle\log\Big\{\sum_{k=1}^K\heta_{P,k}(X)\exp[\theta_k^\top T(X) + \alpha_k'  ]\Big\}\right] - \log(N(\theta, \alpha'))\\
& = O(\theta, \alpha')  - \log(N(\theta, \alpha'))\,.
\end{aligned}
\] 
Defining $F_2 = \{c(\theta, \alpha')\mid (\theta, \alpha')\in \reals^q\}$ we see that $F_1 = F_2$, whose argument follows.  We first notice that $F_2 \subset F_1$ since for any $(\theta , \alpha')\in \reals^q $ it holds $ c(\theta , \alpha') \in F_1$. We also notice that $F_1 \subset F_2$. This is argued by noticing the following: if $(\theta, \alpha') \in F_1$ then $N(\theta,\alpha') = 1$ which implies $c(\theta, \alpha') = (\theta, \alpha')$. 

Here, we summarize the crux of the proof. 
Though there are multiple $(\theta, \alpha')$ that produces the same value of $c(\theta, \alpha')$, each of these $(\theta, \alpha')$'s produce the same value for the objective
\[
O(\theta, \alpha')  - \log(N(\theta, \alpha')) = O(c(\theta, \alpha'))\,,
\] and $c(\theta, \alpha')$ always satisfy the constraint. So, the optimal point $(\theta, \alpha)$ of \eqref{eq:objs1} corresponds to multiple $(\theta, \alpha')$'s and each of them maximizes $O(\theta, \alpha')  - \log(N(\theta, \alpha'))$. Furthermore, we obtain the optimal  $(\theta, \alpha)$ from any of $(\theta, \alpha')$ (which optimizes $O(\theta, \alpha')  - \log(N(\theta, \alpha'))$) using the transformation $c(\theta, \alpha')$. 

The mathematical description of the change of variable follows.
With the change of variable we rewrite \eqref{eq:objs1} as
\[
\begin{aligned}
& (\hat \theta , \hat \alpha) =  c(\hat \theta , \hat \alpha'), ~ (\hat \theta , \hat \alpha')= \argmax_{(\theta, \alpha')} \left\{ O(c(\theta, \alpha'))\mid  N(c(\theta, \alpha')) = 1, (\theta, \alpha')\in \reals^q  \right\}\\
\text{or }& (\hat \theta , \hat \alpha) =  c(\hat \theta , \hat \alpha'), ~ (\hat \theta , \hat \alpha')  = \argmax_{(\theta, \alpha')} \left\{ O(\theta, \alpha') - \log(N(\theta, \alpha')) \mid  (\theta, \alpha') \in \reals^q \right\}
\end{aligned}
\] where the constraint disappear because $N(c(\theta, \alpha')) = 1$ for any $(\theta, \alpha') \in \reals^q$. This completes the proof. 


\subsection{Proof of proposition \ref{prop:anchor-points}}
\label{supp:prop:anchor-points}

Suppose there are two sets of tilt parameters $(\theta_k, \alpha_k)$'s and $(\eta_k, \beta_k)$'s that satisfy \eqref{eq:distribution-matching}:
\[
q_X\{x\} = \sum_{k=1}^Kp\{x,Y=k\}\exp(\theta_k^\top T(x)+ \alpha_k) = \sum_{k=1}^Kp\{x,Y=k\}\exp(\eta_k^\top T(x) +\beta_k).
\]
For any $x\in\cS_k$, the terms that include $p\{x,Y=l\}$, $l\ne k$ vanish:
\[
p\{x,Y=k\}\exp(\theta_k^\top T(x)+\alpha_k) = p\{x,Y=k\}\exp(\eta_k^\top T(x)+\beta_k).
\]
This implies
\[
\theta_k^\top T(x) + \alpha_k = \eta_k^\top T(x) + \beta_k \text{ for all }x\in\cS_k.
\]
We conclude $\theta_k = \eta_k$ and $\alpha_k = \beta_k$ because $T(\cS_k)$ is  $p$-dimensional, so there are $p$ points $x_1,\dots,x_p\in\cS_k$ such that $T(x_1),\dots,T(x_p)\in\reals^p$ are linearly independent.


\subsection{Proof of Theorem \ref{thm:tilt-concentration}}

For a probabilistic classifier  $\eta: \cX \to \Delta^K$ and the parameter $\coef = (\coef_1^\top , \dots, \coef_K^\top )^\top$
we define the centered logit function $f:\cX \to \reals^K$  as $f_a(x) = \log\{\eta_a(x)\} - \frac 1K \sum_{b = 1}^K \log\{\eta_b(x)\}$. 
We define the functions $u_k(f, \coef) = \eta_a(x) \exp(\coef_k^\top S(x))$, $u_{\cdot} (f, \coef) = \sum_{k=1}^K u_k(f, \coef)$ and $v_k(f, \coef) = u_k(f, \coef)/u_{\cdot}(f, \coef)$, and notice that the objective is
\begin{equation} \textstyle
    \hat L(f, \coef)   = \Ex_{\hQ_X} [\log\{ u_{\cdot}(f, \coef) \}] - \log\{ \Ex_{\hP}[\exp(\coef_y ^\top S(x))] \}\,,
\end{equation}  whereas the true objective is 
\begin{equation} \textstyle
      L^\star( f, \coef)  = \Ex_{Q_X} [\log\{ u_{\cdot}(f, \coef) \}] - \log\{ \Ex_{P}[\exp(\coef^\top_y S(x))] \}\,.
\end{equation}
We see that the first order optimality conditions in estimating $\hat \coef$ are 
\begin{equation} \textstyle \label{eq:optimality-estimation}
\begin{aligned}
 0 & = \partial_{\coef_a} \hat L(\hat f, \hat \coef)\\
 & = \partial_{\coef_a}\big[ \Ex_{\hQ_X} [\log\{ u_{\cdot}(\hat f, \hat \coef) \}] - \log\{ \Ex_{\hP}[\exp( {\hat \coef_y} ^\top S(x))] \}\big]\\
 & = \Ex_{\hQ_X}\big[ \partial_{\coef_a} \{u_{\cdot}(\hat f, \hat \coef)\}/u_{\cdot}(\hat f, \coef) \big] - \frac{\partial_{\coef_a}\{ \Ex_{\hP}[\exp({\hat \coef_y} ^\top S(x))] \}}{ \Ex_{\hP}[\exp({\hat \coef_y} ^\top S(x))]}\\
 & = \Ex_{\hQ_X} [S(x)v_a(\hat f, \hat \coef)] - \Ex_{\hP}[S(x)\exp({\hat \coef_a} ^\top S(x))\bI\{y = a\}] 
\end{aligned}
  \end{equation} where the last inequality holds because $\Ex_{\hP}[\exp(\hcoef_y ^\top S(x))] = 1$. 
Similarly, the first order optimality condition at truth (for $\coef^\star$) are 
 \begin{equation}
     \textstyle \label{eq:optimality-truth}
     \begin{aligned}
     0 &= \partial_{\coef_a}  L^\star( f^\star,  \coef^\star)\\
     &= \Ex_{Q_X} [S(x)v_a( \sf,  \scoef)] - \Ex_{P}[S(x)\exp({\coef_a^\star} ^\top S(x))\bI\{y = a\}]\\
     &= \Ex_{Q_X} [S(x)v_a( \sf,  \scoef)] - \Ex_{P_X}[S(x)\exp({\coef_a^\star} ^\top S(x))\eta_{P, a}^\star(x)]
     \end{aligned}
 \end{equation} We decompose \ref{eq:optimality-estimation} using the Taylor expansion and obtain: 
 \begin{equation} \textstyle
 \label{eq:optimality-theta-hat}
     0 = \partial_{\coef_a} \hL (\sf, \hcoef) + \langle \hf - \sf, \partial_{f} \partial_{\coef_a} \hL( \tf, \hcoef)\rangle 
 \end{equation} where $\tilde f$ is a function in the bracket $[f^\star, \hat f]$, \ie\ for every $x$, $\tilde f(x)$ is a number between $\hat f(x)$ and $f^\star(x)$
\paragraph{Bound on $\langle \hf - \sf, \partial_{f} \partial_{\coef_a} \hL( \tf, \hcoef)\rangle $:} To bound the term we define $\zeta = \hf - \sf$ and  notice that 
\[
\begin{aligned}
  \langle \zeta, \partial_{f} \partial_{\theta_a} \hL( f, \hcoef)\rangle
 & =\textstyle \sum_{b}\langle \zeta_b, \partial_{f_b} \partial_{\theta_a} \hL( f, \hcoef)\rangle\\
 &=\textstyle \sum_{b}\langle \zeta_b, \partial_{f_b} \{\Ex_{\hQ_X} [S(x)v_a(f, \hat \coef)]\} \rangle \\
 &=\textstyle \sum_{b}\Ex_{\hQ_X}[S(x) \zeta_b(x)  \partial_{f_b}\{v_a(f,  \hat \coef) \} ] \\
 &=\textstyle \sum_{b}\Ex_{\hQ_X}[S(x) \zeta_b(x) v_a(f,  \hcoef) (\delta_{ab} - v_a(f, \hcoef))  ]\,.
\end{aligned}
\] 
The derivative in third equality in the above display is calculated in Lemma \ref{lemma:derivatives}. Here, from Assumption \ref{assump:bounded-covariate} we notice that  $S(x)$ is bounded, \ie, there exists a $c_1>0$ such that $\|S(x)\|_2 \le c_1$ for all $x \in \cX$. This implies the followings:   we have
\[
\begin{aligned} 
\textstyle
\|\Ex_{\hQ_X}[S(x) \zeta_b(x) v_a(f,  \coef) (\delta_{ab} - v_a(f, \coef))  ]\|_2
 \le c_1\Ex_{\hQ_X}[ |\zeta_b(x)| ]
\end{aligned}
\] since $0 \le v_a(f,  \coef) (\delta_{ab} - v_a(f, \coef))\le 1$ and we have 
\[
\begin{aligned}
\textstyle
\|\langle \zeta, \partial_{f} \partial_{\theta_a} \hL( f, \hcoef)\rangle\|_2 \le \sum_b c_1\Ex_{\hQ_X}[ |\zeta_b(x)| ] \le c_1 \sqrt{K} \Ex_{\hQ_X}[ \|\zeta(x)\|_2 ]\,.
\end{aligned}
\]
It follows from Assumption \ref{assmp:source-classifier}: with probability at least $ 1- \delta$ it holds   $\sup_{i \in[n_Q]} \|\hf(x_{Q, i}) - \sf(x_{Q,i})\|_2 \le c_2 r_{n_P} \sqrt{\log(n_Q)\log\{1/\delta\}}$, we conclude that  
\begin{equation}
    |\langle \hf - \sf, \partial_{f} \partial_{\theta_a} \hL( f, \theta)\rangle| \le \sqrt{K} c_1 c_2 r_{n_P} \sqrt{\log(n_Q)\log\{1/\delta\}}
\end{equation} holds with probability at least $ 1- \delta$. 

\paragraph{The term $\partial_{\coef_a} \hL(\sf, \hcoef)$}
We use strong convexity \ref{assmp:strong-convexity} and convergence of the loss that \[\sup_{\xi \in K}|\hat L(f^\star, \coef) -   L^\star(f^\star, \coef)| \stackrel{n_P, n_Q \to \infty}{\longrightarrow} 0\] for any compact $K$ in \citet[Corollary 3.2.3.]{vaart2000Weak} to conclude that $\hcoef \to \scoef$ in probability and hence $\hcoef$ is a consistent estimator for $\scoef$.

Following the consistency of $\hcoef$  we see that for sufficiently large $n_P, n_Q$ we have $\|\hcoef - \scoef\|_2 \le \delta_\coef$ ($\delta_\coef$ is chosen according to Lemma \ref{lemma:lb-normalizer}) with probability at least $ 1 -\delta$ and on the event it holds: $\|\hcoef\|_2 \le \|\scoef\|_2 + \delta_\coef$. We define empirical process 
\begin{equation}\label{eq:emp-1st-derivative}
    Z_{a, n_P, n_Q} = \sup_{\|\coef\|_2\le \|\scoef\|_2 + 1} \|\partial_{\coef_a}\hL(\sf, \coef) - \partial_{\coef_a}\sL(\sf, \coef)\|_2
\end{equation}
for which we shall provide a high probability upper bound. We denote $Z_{a, n_P, n_Q}(\coef) = \partial_{\coef_a}\hL(\sf, \coef) - \partial_{\coef_a}\sL(\sf, \coef)$ and notice that 
\[
\begin{aligned}
 & \partial_{\coef_a} \hL(\sf, \coef) - \partial_{\coef_a}\sL(\sf, \coef)\\
& = \underbrace{\Ex_{\hQ_X} [S(x)v_a(\sf,  \coef)] - \Ex_{Q_X} [S(x)v_a(\sf,  \coef)]}_{\triangleq \bA(\coef)}\\
& ~~~~ \underbrace{ -  \frac{\Ex_{\hP}[S(x)\exp( \coef_a ^\top S(x))\bI\{y = a\}]}{\Ex_{\hP}[\exp( \coef_y ^\top S(x))]} +  \frac{\Ex_{P}[S(x)\exp( \coef_a ^\top S(x))\bI\{y = a\}]}{\Ex_{P}[\exp( \coef_y ^\top S(x))]} }_{\triangleq \bB(\coef)} = \bA(\coef) + \bB(\coef)
\end{aligned}
\]
where to bound $\bA(\coef)$ we notice that $S(x_{Q, i})v_a(\sf,  \coef)$ are \IID\ and bounded by $c_1$ ($\|S(x)v_a(\sf,  \coef)\| \le \|S(x)\|_2 \le c_1$ for all $x\in \cX$) and hence sub-gaussian. We apply Hoeffding's concentration inequality for a sample mean of \IID\  sub-gaussian random variables and obtain a constant $c_2>0$ such that for any $\delta>0$ with probability at least $1 - \delta$ it holds
\[ \textstyle
\bA(\coef) = \Ex_{\hQ_X} [S(x)v_a(\sf,  \coef)] - \Ex_{Q_X} [S(x)v_a(\sf,  \coef)] \le c_1 c_2 \sqrt{\frac{\log (1/\delta)}{n_Q}}\,.
\] Using a chaining argument over an $\ell_2$ ball of radius $\|\scoef\|_2 + \delta_\coef$ we  obtain a uniform bound as the following: there exists a constant $c_2>0$ such that for any $\delta>0$ with probability at least $1 - \delta$ it holds
\begin{equation}
    \textstyle \label{eq:bound.A}
\sup_{\coef: \|\coef \|_2 \le \|\scoef\|_2 + \delta_\coef}\bA(\coef)  \le c_1 c_3  \sqrt{\frac{K(p+1)\log (1/\delta)}{n_Q}}\,.
\end{equation}
To bound $\bB$ we first define 
\[
\begin{aligned}
\bB.1(\coef, n_P) \triangleq & ~\Ex_{\hP}[S(x)\exp( \coef_a ^\top S(x))\bI\{y = a\}] - \Ex_{P}[S(x)\exp( \coef_a ^\top S(x))\bI\{y = a\}]\\
\bB.2(\coef, n_P) \triangleq &~ \Ex_{\hP}[\exp( \coef_a ^\top S(x))] - \Ex_{P}[\exp( \coef_a ^\top S(x))]
\end{aligned}
\]
and  notice that both the random variables $\{ S(x_{P, i})\exp( \coef_a ^\top S(x_{P, i}))\bI\{y_{P, i} = a\} \}_{i = 1}^{n_P}$ and $\{ \exp( \coef_a ^\top S(x_{P, i})) \}_{i = 1}^{n_P}$ are bounded for all $\|\coef \|_2 \le \|\scoef\|_2 + \delta_\coef$. Similarly as before we obtain constant $c_4, c_5 > 0$ such that the following hold with probability at least $1- \delta$:
\begin{equation} \label{eq:sup-con-b1b2}
    \begin{aligned}
 \sup_{\coef: \|\coef \|_2 \le \|\scoef\|_2 + \delta_\coef} \big|\bB.1(\coef, n_P)\big|
& \le\textstyle c_4  \sqrt{\frac{K(p+1)\log (1/\delta)}{n_P}}\\
 \sup_{\coef: \|\coef \|_2 \le \|\scoef\|_2 + \delta_\coef} \big|\bB.2(\coef, n_P)\big|
& \le\textstyle c_5  \sqrt{\frac{K(p+1)\log (1/\delta)}{n_P}}\,.
\end{aligned}
\end{equation}
In Lemma \ref{lemma:lb-normalizer} we notice that
\begin{equation}\label{eq:lb-normalizer}
    \inf_{\coef: \|\coef \|_2 \le \|\scoef\|_2 + \delta_\coef} \Ex_{P}[\exp( \coef_y ^\top S(x))] \ge \frac 12\,.
\end{equation}
Gathering all the inequalities  in $\bB$ we obtain
\[
\begin{aligned}
\bB(\coef) & = -  \frac{\Ex_{\hP}[S(x)\exp( \coef_a ^\top S(x))\bI\{y = a\}]}{\Ex_{\hP}[\exp( \coef_y ^\top S(x))]} +  \frac{\Ex_{P}[S(x)\exp( \coef_a ^\top S(x))\bI\{y = a\}]}{\Ex_{P}[\exp( \coef_y ^\top S(x))]}\\
& = -  \frac{\Ex_{P}[S(x)\exp( \coef_a ^\top S(x))\bI\{y = a\}] +\bB.1(\coef, n_P) }{\Ex_{P}[\exp( \coef_y ^\top S(x))] + \bB.2(\coef, n_P)} +  \frac{\Ex_{P}[S(x)\exp( \coef_a ^\top S(x))\bI\{y = a\}]}{\Ex_{P}[\exp( \coef_y ^\top S(x))]}\\
& = \frac{-\bB.1(\coef, n_P)\Ex_{P}[\exp( \coef_a ^\top S(x))] + \bB.2(\coef, n_P)\Ex_{P}[S(x)\exp( \coef_a ^\top S(x))\bI\{y = a\}]}{\Big\{\Ex_{P}[\exp( \coef_y ^\top S(x))] + \bB.2(\coef, n_P)\Big\}\Ex_{P}[\exp( \coef_y ^\top S(x))]}
\end{aligned}
\] and this implies 
\[
\begin{aligned}
|\bB(\coef)| & \le  \frac{\|\bB.1(\coef, n_P)\|_2\Ex_{P}[\exp( \coef_a ^\top S(x))] + |\bB.2(\coef, n_P)|\|\Ex_{P}[S(x)\exp( \coef_a ^\top S(x))\bI\{y = a\}]\|_2}{\Big\{\Ex_{P}[\exp( \coef_y ^\top S(x))] - |\bB.2(\coef, n_P)|\Big\}\Ex_{P}[\exp( \coef_y ^\top S(x))]}
\end{aligned}
\] where we use \eqref{eq:sup-con-b1b2} and \eqref{eq:lb-normalizer} to obtain a constant $c_6>0$ such that with probability at least $1 - \delta$ it holds
\begin{equation}\textstyle\label{eq:bound.B}
    \sup_{\coef: \|\coef \|_2 \le \|\scoef\|_2 + \delta_\coef} |\bB(\coef)| \le c_6 \sqrt{\frac{K(p+1)\log (1/\delta)}{n_P}}\,.
\end{equation}
We now combine \eqref{eq:bound.A} and \eqref{eq:bound.B} and obtain a constant $c_7>0$ such that with probability at least $1 - 2\delta$ we have 
\begin{equation} \textstyle\label{eq:emp-ub}
    Z_{a, n_P, n_Q} \le c_7 \left\{\sqrt{\frac{K(p+1)\log (1/\delta)}{n_P}} + \sqrt{\frac{K(p+1)\log (1/\delta)}{n_Q}}\right\}\,.
\end{equation}

Returning to the first order optimality condition \eqref{eq:optimality-theta-hat} for estimating $\hcoef$ 
we notice that 
\[
\begin{aligned}
 0 & = \textstyle\sum_{a} (\hcoef_a - \scoef_a)^\top \big\{\partial_{\coef_a} \hL (\sf, \hcoef) + \langle \hf - \sf, \partial_{f} \partial_{\coef_a} \hL( \tf, \hcoef)\rangle\big\}\\
 & = \textstyle\sum_{a} (\hcoef_a - \scoef_a)^\top \partial_{\coef_a} \sL (\sf, \hcoef) + \sum_{a} (\hcoef_a - \scoef_a)^\top  \big\{Z_{a, n_P, n_Q}(\hcoef) + \langle \hf - \sf, \partial_{f} \partial_{\coef_a} \hL( \tf, \hcoef)\rangle\big\}\\
 & = \textstyle(\hcoef - \scoef) ^\top \partial_{\coef} \sL (\sf, \hcoef) + \sum_{a} (\hcoef_a - \scoef_a)^\top  \big\{Z_{a, n_P, n_Q}(\hcoef) + \langle \hf - \sf, \partial_{f} \partial_{\coef_a} \hL( \tf, \hcoef)\rangle\big\}
\end{aligned}
\]

We combine it with the first ordder optimality condition for $\scoef$ \eqref{eq:optimality-truth}
to obtain 
\[
\begin{aligned}
(\hcoef - \scoef) ^\top \big\{\partial_{\coef} \sL (\sf, \hcoef) - \partial_{\coef} \sL (\sf, \scoef)\big\}\\
+ \textstyle\sum_{a} (\hcoef_a - \scoef_a)^\top  \big\{Z_{a, n_P, n_Q}(\hcoef) + \langle \hf - \sf, \partial_{f} \partial_{\coef_a} \hL( \tf, \hcoef)\rangle\big\} = 0
\end{aligned}
\]
which can be rewritten as 
\begin{equation} \label{eq:eqA.24}
  \textstyle  (\hcoef - \scoef) ^\top \big\{\partial_{\coef} \sL (\sf, \hcoef) - \partial_{\coef} \sL (\sf, \scoef)\big\} = - \sum_{a} (\hcoef_a - \scoef_a)^\top  \big\{Z_{a, n_P, n_Q}(\hcoef) + \langle \hf - \sf, \partial_{f} \partial_{\coef_a} \hL( \tf, \hcoef)\rangle\big\}\,.
\end{equation} Using the strong convexity assumption at $\scoef$ we obtain that the left hand side in the above equation is lower bounded as 
\begin{equation} \label{eq:lhs}
    (\hcoef - \scoef) ^\top \big\{\partial_{\coef} \sL (\sf, \hcoef) - \partial_{\coef} \sL (\sf, \scoef)\big\} \ge \strongconvexity \|\hcoef - \scoef\|_2^2 \,.
\end{equation}

Let $\cE$ be the event on which the following hold:
\begin{enumerate}
    \item $\|\hcoef - \scoef\|_2 \le \delta_\coef$, 
    \item $|\langle \hf - \sf, \partial_{f} \partial_{\coef_a} \hL(\tilde  f, \hat  \coef)\rangle| \le \sqrt{K}c_1c_2 r_{n_P} \sqrt{\log(n_Q)\log\{1/\delta\}}$ for all $a$, 
    \item $Z_{a, n_P, n_Q} \le c_7\Big\{\sqrt{\frac{ K(p+1) \log(1/\delta)}{n_P}} + \sqrt{\frac{K(p+1) \log(1/\delta)}{n_Q}}\Big\}$ for all $a$. 
\end{enumerate}
We notice that the event $\cE$ has probability $1 - (2 K + 1)\delta$. 
Under the event there exists a  $c_9>0$ such that  the right hand side in \eqref{eq:eqA.24} is upper bounded as 
\begin{equation} \label{eq:rhs}
    \begin{aligned}
& \textstyle\left | - \sum_{a} (\hcoef_a - \scoef_a)^\top  \big\{Z_{a, n_P, n_Q}(\hcoef) + \langle \hf - \sf, \partial_{f} \partial_{\coef_a} \hL( \tf, \hcoef)\rangle\big\}\right|\\
& \le \textstyle\sum_{a} \|\hcoef_a - \scoef_a\|_2 \Big\{\|Z_{a, n_P, n_Q}(\hcoef)\|_2 + \big\|\langle \hf - \sf, \partial_{f} \partial_{\coef_a} \hL( \tf, \hcoef)\rangle\big\|_2 \Big\}\\
& \le\textstyle \sum_{a} \|\hcoef_a - \scoef_a\|_2 \Big\{Z_{a, n_P, n_Q} + \big\|\langle \hf - \sf, \partial_{f} \partial_{\coef_a} \hL( \tf, \hcoef)\rangle\big\|_2 \Big\}\\
& \le \textstyle \sum_{a} \|\hcoef_a - \scoef_a\|_2 c_{9}  \Big\{r_{n_P} \sqrt{\log(n_Q)\log\{1/\delta\}} + \sqrt{\frac{ K(p+1)\log(1/\delta)}{n_P}} + \sqrt{\frac{K(p+1) \log(1/\delta)}{n_Q}}\Big\} \\
& \le \textstyle c_{9}  \Big\{r_{n_P} \sqrt{\log\{n_Q/\delta\}} + \sqrt{\frac{ K(p+1)\log(1/\delta)}{n_P}} + \sqrt{\frac{K(p+1) \log(1/\delta)}{n_Q}}\Big\} \sqrt{K} \|\hcoef - \scoef\|_2 
\end{aligned}
\end{equation}
 Combining the bounds \eqref{eq:lhs} and \eqref{eq:rhs} for left and right hand sides we obtain a $c_{10}>0$ such that on the event $\cE$ it holds 
 \[\textstyle\|\hcoef - \scoef\|_2 \le c_{10}\Big\{r_{n_P} \sqrt{\log(n_Q)\log\{1/\delta\}} + \sqrt{\frac{K(p+1) \log(1/\delta)}{n_P}} + \sqrt{\frac{K(p+1) \log(1/\delta)}{n_Q}}\Big\}\,. \]
 
 Having a concentration on $\hcoef$ we now notice that \[
 \begin{aligned}
 & \|\hat \weight - \weight^\star\|_{1, P} \\
 & = \textstyle\int \big|\exp(\hcoef_y^\top S(x)) - \exp({\scoef_y}^\top S(x))\big| p_X(x)dx\\
 & = \textstyle\int |(\hcoef_y - \scoef_y)^\top S(x)| \exp(\coef_x) p_X(x)dx
 \end{aligned}
 \]
 where $\coef_x$ is a number between $\hcoef_y^\top S(x)$ and ${\scoef_y}^\top S(x)$. On the event $\cE$ we notice that $\|\hcoef\|_2 \le \|\scoef\|_2 + \delta_\coef$ and hence it holds: $|\hcoef_y^\top S(x)|\le \|\hcoef_y\|_2 \|S(x)\|_2 \le \|\hcoef\|_2 \|S(x)\|_2 \le c_1 (\|\scoef\|_2 + \delta_\coef)$. Furthermore, we notice that $|{\scoef_y}^\top S(x)|\le \|\scoef\|_2 \|S(x)\|_2 \le c_1\|\scoef\|_2$, which implies $|\coef_x| \le c_1 (\|\scoef\|_2 + \delta_\coef)$. Returning to the integral we obtain that on the event $\cE$ it holds: 
 \[
 \begin{aligned}
 &\textstyle \int |(\hcoef_y - \scoef_y)^\top S(x)| \exp(\coef_x) p_X(x)dx\\
 & \le\textstyle \|\hcoef_y - \scoef_y\|_2 \int \|S(x)\|_2\exp\{c_1 (\|\scoef\|_2 + \delta_\coef)\} p_X(x)dx \\
 & \le \|\hcoef - \scoef\|_2 c_1 \exp\{c_1 (\|\scoef\|_2 + \delta_\coef)\}\\
 & \le\textstyle c_{11}\Big\{r_{n_P} \sqrt{\log(n_Q)\log\{1/\delta\}} + \sqrt{\frac{K(p+1) \log(1/\delta)}{n_P}} + \sqrt{\frac{K(p+1) \log(1/\delta)}{n_Q}}\Big\}
 \end{aligned}
 \] for some $c_{11}>0$, which holds with probability at least $ 1- (2K + 1)\delta$. 

\begin{lemma}
\label{lemma:lb-normalizer}
There exists $\delta_\xi>0$ such that \[
\inf_{\coef: \|\coef \|_2 \le \|\scoef\|_2 + \delta_\coef} \Ex_{P}[\exp( \coef_y ^\top S(x))] \ge \frac 12\,.
\]
\end{lemma}

\begin{proof}[Proof of Lemma \ref{lemma:lb-normalizer}]
  To establish a bound on $\Ex_{P}[\exp( \coef_y ^\top S(x))] - \Ex_{P}[\exp( {\scoef_y} ^\top S(x))]$ for any $\|\coef - \scoef\|_2 \le \delta_\coef$ we notice that \[
 \begin{aligned}
 & |\Ex_{P}[\exp( \coef_y ^\top S(x))] - \Ex_{P}[\exp( {\scoef_y} ^\top S(x))]| \\
 & = \textstyle\int \big|\exp(\coef_y^\top S(x)) - \exp({\scoef_y}^\top S(x))\big| p_X(x)dx\\
 & = \textstyle\int |(\coef_y - \scoef_y)^\top S(x)| \exp(\coef_x) p_X(x)dx
 \end{aligned}
 \]
 where $\coef_x$ is a number between $\coef_y^\top S(x)$ and ${\scoef_y}^\top S(x)$. We notice that $\|\coef\|_2 \le \|\scoef\|_2 + \delta_\coef$ and hence it holds: $|\coef_y^\top S(x)|\le \|\coef_y\|_2 \|S(x)\|_2 \le \|\coef\|_2 \|S(x)\|_2 \le c_1 (\|\scoef\|_2 + \delta_\coef)$. Furthermore, we notice that $|{\scoef_y}^\top S(x)|\le \|\scoef\|_2 \|S(x)\|_2 \le c_1\|\scoef\|_2$, which implies $|\coef_x| \le c_1 (\|\scoef\|_2 + \delta_\coef)$. Returning to the integral we obtain 
 \[
 \begin{aligned}
 &\textstyle \int |(\coef_y - \scoef_y)^\top S(x)| \exp(\coef_x) p_X(x)dx\\
 & \le\textstyle \|\coef_y - \scoef_y\|_2 \int \|S(x)\|_2\exp\{c_1 (\|\coef\|_2 + \delta_\coef)\} p_X(x)dx \\
 & \le \delta_\coef c_1 \exp\{c_1 (\|\scoef\|_2 + \delta_\coef)\}\,.
 \end{aligned}
 \] We choose $\delta_\coef>0$ small enough such that $\delta_\coef c_1 \exp\{c_1 (\|\scoef\|_2 + \delta_\coef)\} \le 1/2$.  Since $\Ex_{P}[\exp( {\scoef_y} ^\top S(x))] = 1$ we obtain that for any $\coef$ with $\|\coef - \scoef\|_2 \le 1/2$ we have 
 \[
 \begin{aligned}
 \textstyle\Ex_{P}[\exp( \coef_y ^\top S(x))] \ge \Ex_{P}[\exp( {\scoef_y} ^\top S(x))] - |\Ex_{P}[\exp( \coef_y ^\top S(x))] - \Ex_{P}[\exp( {\scoef_y} ^\top S(x))]| \ge \frac 12\,.
 \end{aligned}
 \] This implies the lemma.

\end{proof}

\begin{lemma}[Derivatives] \label{lemma:derivatives}
The following holds:
\begin{enumerate}
    \item $\partial_{\theta_b} u_a(f, \theta) = T(x)u_a(f, \theta) \delta_{a,b}$, $\partial_{\theta_b} u_{\cdot}(f, \theta) = T(x)u_b(f, \theta) $ and $\partial_{\theta_b} v_a(f, \theta) = T(x)v_a(f, \theta)\{\delta_{a, b} - v_b(f, \theta)\}$. 
    \item $\partial_{f_b} \eta_a = \eta_a (\delta_{a, b} - \eta_b)$, $\partial_{f_b} \{u_a(f, \theta)\} = (\delta_{a, b} - \eta_b)u_a(f, \theta)$, $\partial_{f_b} \{u_{\cdot}(f, \theta)\} = u_b(f, \theta) - \eta_b u_{\cdot}(f, \theta)$ and $\partial_{f_b} \{v_a(f, \theta)\} = v_a(f, \theta) (\delta_{a, b} - v_b(f, \theta))$. 
\end{enumerate}
\end{lemma}

\begin{proof}[Proof of lemma \ref{lemma:derivatives}]
We calculate the derivatives one by one. 
\paragraph{1:} We notice that 
\[
\begin{aligned}
\partial_{\theta_b} u_a(f, \theta)  &= \partial_{\theta_b} \{\eta_a(x) \exp(\theta_a^\top T(x))\}\\
& = \eta_a(x) \exp(\theta_a^\top T(x)) T(x) \delta_{a, b} = T(x)u_a(f, \theta) \delta_{a,b}
\end{aligned}
\] and that 
\[
\textstyle\partial_{\theta_b} u_{\cdot}(f, \theta) = \sum_{a = 1}^K T(x)u_a(f, \theta) \delta_{a,b} = T(x)u_b(f, \theta)
\] which finally implies
\[
\begin{aligned}
\partial_{\theta_b} v_a(f, \theta) &=\textstyle \frac{\partial_{\theta_b} \{u_a(f, \theta)\} u_{\cdot }(f, \theta) - \partial_{\theta_b} \{u_{\cdot}(f, \theta)\} u_{a }(f, \theta) }{\{u_{\cdot} (f, \theta)\}^2}\\
& =\textstyle \frac{T(x)u_a(f, \theta) \delta_{a,b} u_{\cdot }(f, \theta) - T(x)u_b(f, \theta) u_{a }(f, \theta) }{\{u_{\cdot} (f, \theta)\}^2}\\
& = T(x) \{u_a / u_{\cdot}\} \{\delta_{a, b} - u_a / u_{\cdot}\} = T(x) v_a (\delta_{a, b} - v_a)
\end{aligned}
\]

\paragraph{2:} Here \[
\begin{aligned}
\partial_{f_b} \eta_a  = \textstyle\partial_{f_b} \left\{\frac{e^{f_a}}{\sum_{j} e^{f_j}} \right\} = \textstyle\frac{\delta_{a, b} e^{f_a}\sum_{j} e^{f_j} - e^{f_a} e^{f_b} }{\Big\{\sum_{j} e^{f_j}\Big\}^2} = \eta_a (\delta_{a, b} -\eta_b)\,,
\end{aligned}
\] 
\[
\begin{aligned}
\partial_{f_b} \{u_a(f, \theta)\} & = \partial_{f_b} \{\eta_a\} \exp(\theta_a^\top T(x)) \\
& = \eta_a (\delta_{a, b} -\eta_b)\exp(\theta_a^\top T(x)) = (\delta_{a, b} -\eta_b)u_a(f, \theta)\,,
\end{aligned}
\]
\[ \textstyle
\partial_{f_b} \{u_{\cdot}(f, \theta)\}  = \sum_{a} (\delta_{a, b} -\eta_b)u_a(f, \theta) = u_{b}(f, \theta) - \eta_b u_{\cdot}(f, \theta)
\] and finally,
\[
\begin{aligned}
\partial_{f_b} \{v_a(f, \theta)\} & = \frac{ \{\partial_{f_b} u_a(f, \theta)\} u_{\cdot}(f, \theta) -  \{\partial_{f_b} u_{\cdot}(f, \theta)\} u_{a}(f, \theta) }{\{u_{\cdot}(f, \theta)\}^2}\\
& = \frac{ (\delta_{a, b} -\eta_b)u_a(f, \theta) u_{\cdot}(f, \theta) -  \{u_{b}(f, \theta) - \eta_b u_{\cdot}(f, \theta)\} u_{a}(f, \theta) }{\{u_{\cdot}(f, \theta)\}^2}\\
& = (u_a / u_{\cdot}) \{ \delta_{a, b} u_{\cdot} - \cancel{\eta_b u_{\cdot}} - u_a  +\cancel{\eta_b u_{\cdot}}  \}/ u_{\cdot}\\
& =  (u_a / u_{\cdot}) \{ \delta_{a, b}  - (u_a/u_{\cdot})   \} = v_a \{ \delta_{a, b} - v_a \}
\end{aligned}
\]
\end{proof}

\subsection{Proof of Lemma \ref{lemma:generalization-bound}}

We start by decomposing the loss difference in the left hand side of equation \eqref{eq:gen-bound}. 
\begin{equation}
    \begin{aligned}
    \cL_Q(\hat f_{\hat \weight}) - \cL_Q(f^\star) = \cL_P(\hat f, \weight^\star) - \cL_P(f^\star, \weight^\star)\\
    = \textstyle\underbrace{\cL_P(\hat f, \weight^\star) - \cL_{\hP}(\hat f, \weight^\star)}_{(a)} + \underbrace{\cL_{\hP}(\hat f, \weight^\star) - \cL_{\hP}(\hat f, \hat \weight)}_{(b)} + \underbrace{\cL_{\hP}(\hat f, \hat \weight) - \cL_{\hP}( f^*, \hat \weight)}_{\le 0}\\
    \textstyle + \underbrace{\cL_{\hP}( f^*, \hat \weight) - \cL_{\hP}( f^*,  \weight^\star)}_{(c)} + \underbrace{\cL_{\hP}( f^*,  \weight^\star) - \cL_P( f^*,  \weight^\star)}_{(d)}\,,
    \end{aligned}
\end{equation} where we write $\hat f_{\hat \weight} \equiv \hat f$. 

\paragraph{Uniform bound on (a)} To control (a) in \eqref{eq:gen-bound} we establish a concentration bound on the following generalization error
\[
\begin{aligned}
\sup_{f \in \cF} \big\{ \cL_P( f, \weight^\star) - \cL_{\hP}( f, \weight^\star) \big\} \\
=\textstyle \sup_{f \in \cF} \Big\{ \Ex\big[g_f(X, Y)\big] -\frac1 {\nsource} \sum_{i = 1}^{\nsource} g_f(X_{P, i}, Y_{P, i}) \Big\}
\triangleq F(z_{1:\nsource})\,,
\end{aligned}
\] where, for $i \ge 1$ we denote $z_{1:i} = (z_1, \dots, z_i)$ and  $z_i = (X_{P, i}, Y_{P, i})$.  First, we use a modification of McDiarmid concentration inequality to bound $F(z_{1:\nsource})$
in terms of its expectation and a $O(1/\sqrt{\nsource})$ term, as elucidated in the following lemma. 
\begin{lemma}
\label{lemma:mcdiarmid}
There exists a constant $c_1>0$ such that with probability at least $1-\delta$ the following holds
\begin{equation}
\label{eq:conc-mcdirmid}
  \textstyle  F(z_{1:\nsource}) \le \Ex\big[F(z_{1:\nsource})\big] + c_1\sqrt{\frac{\log(1/\delta)}{\nsource}}\,.
\end{equation}
\end{lemma}
Next, we use a symmetrization argument (see \cite[Chapter 2, Lemma 2.3.1]{wellner2013weak} ) to bound the expectation $\Ex\big[F(z_{1:\nsource})\big]$ by the Rademacher complexity of the hypothesis class $\cG$, \ie, 
\begin{equation}
\label{eq:symmetrization}
    \Ex\big[F(z_{1:\nsource})\big] \le 2 \cR_{\nsource} (\cG)\,.
\end{equation}
Combining \eqref{eq:conc-mcdirmid} and \eqref{eq:symmetrization} we obtain
\begin{equation}
  \textstyle  (a) = \cL_P(\hat f, \weight^\star) - \cL_{\hP}(\hat f, \weight^\star) \le 2 \cR_{\nsource}(\cG) +   c_1\sqrt{\frac{\log(1/\delta)}{\nsource}}\,,\label{eq:bound-a}
\end{equation} with probability at least $ 1- \delta$. 

\paragraph{Bound on (b) and (c)} Denoting $z_i = (X_{P, i}, Y_{P, i})$ and $\ell_f(z_i) = \ell(f(X_{P,i}), Y_{P, i})$ we  notice that 
for any $f\in \cF$ we have 
\[
\begin{aligned}
  \textstyle ~~\big|\cL_{\hP}( f, \weight^\star) - \cL_{\hP}( f, \hat \weight)\big|
  & = \textstyle\Big|\frac1{\nsource}\sum_{i = 1}^{\nsource}\big\{ \hat \weight(z_i) - \weight^\star (z_i) \big\} \ell_f(z_i)\Big|\\
  & \le\textstyle \frac1{\nsource}\sum_{i = 1}^{\nsource}\big|\big\{ \hat \weight(z_i) - \weight^\star (z_i) \big\} \ell_f(z_i)\big| \le\textstyle \frac{\|\ell\|_\infty}{\nsource}\sum_{i = 1}^{\nsource}\big| \hat \weight(z_i) - \weight^\star (z_i) \big|\,.
\end{aligned}
\] Since $\hat \weight(z) - \weight^\star (z)$ is a sub-gaussian random variable, we use sub-gaussian concentration to establish that for some constant $c_2>0$ \begin{equation}
  \textstyle  \text{for any }f \in \cF, ~~\big|\cL_{\hP}( f, \weight^\star) - \cL_{\hP}( f, \hat \weight)\big| \le \|\ell\|_\infty \Big\{ \Ex_{z_1}\big[|\hat w(z_1) - w^\star (z_1)|\big] + c_2 \sqrt{\frac{\log(1/\delta)}{\nsource}} \Big\} \label{eq:bound-bc}
\end{equation}
 with probability at least $1 - \delta$. This provides a simultaneous bound (on the same probability event) for both (b) and (c) with $f = \hat f$ and $f = f^\star$. 

\paragraph{Bound on (d)} We note that 
\[
\begin{aligned}
 \textstyle ~~ \cL_{\hP}( f^*,  \weight^\star) - \cL_P( f^*,  \weight^\star) = \frac{1}{\nsource}\sum_{i = 1}^{\nsource} \weight^\star (z_i) \ell_{f^\star} (z_i) - \Ex_P\big\{\weight^\star (z_1) \ell_{f^\star} (z_1)\big\}\,,
\end{aligned}
\] where $\{\weight^\star (z_i) \ell_{f^\star} (z_i)\}_{i = 1}^{\nsource}$ are \IID\ sub-gaussian random variables. Using Hoeffding concentration bound we conclude that there exists a constant $c_3>0$ such that for any $\delta>0$ the following holds with probability at least $1 - \delta$
\begin{equation}
\label{eq:bound-d}
 \textstyle   \frac{1}{\nsource}\sum_{i = 1}^{\nsource} \weight^\star (z_i) \ell_{f^\star} (z_i) - \Ex_P\big\{\weight^\star (z_1) \ell_{f^\star} (z_1)\big\} \le c_3 \sqrt{\frac{\log(1/\delta)}{\nsource}}\,.
\end{equation} 
Finally, using \eqref{eq:bound-a} on (a) (which is true on an event of probability $\ge 1 - \delta$), \eqref{eq:bound-bc} on (b) and (c) (simultaneously true on an event of probability $\ge 1- \delta$), and \eqref{eq:bound-d} on (d) (holds on an event of probability $\ge 1- \delta$) we conclude that with probability at least $1 - 3\delta$ the following holds
\begin{equation}
   \textstyle \cL_Q(\hat f_{\hat w}) -  \cL_Q(f^\star) \le  2 \cR_n(\cG) + \|\ell\|_\infty \cdot \Ex_{z_1}\big[|\hat \weight(z_1) - \weight^\star (z_1)|\big] + c_4 \sqrt{\frac{\log(1/\delta)}{\nsource}} 
\end{equation} where $c_4 = c_1 + \|\ell\|_\infty c_2 + c_3$.

\begin{proof}[Proof of Lemma \ref{lemma:mcdiarmid}] For the simplicity of notations we drop the subscript from $\nsource$ and denote the sample size simply by $n$. 
For $i \le n$ we define $\Ex_{i:n}$ as the expectation with respect to the random variables $z_i, \dots, z_n$, and for $i > n$ we define $\Ex_{i:n}\big[F(z_{1:n})\big] = F(z_{1:n})$  and notice that
\begin{equation}
\label{eq:martingale-sum}
  \textstyle  F(z_{1:n}) - \Ex\big[F(z_{1:n})\big] = \sum_{i = 1}^{n} \Big\{ \Ex_{(i+1):n}\big[F(z_{1:n})\big] - \Ex_{{i:n}}\big[F(z_{1:n})\big] \Big\}\,.
\end{equation}
Here, 
\begin{equation}
\label{eq:martingale-diff}
    \begin{aligned}
& ~~ \Ex_{(i+1):n}\big[F(z_{1:n})\big] - \Ex_{{i:n}}\big[F(z_{1:n})\big]\\ 
& = \Ex_{(i+1):n} \Big\{F(z_{1:n}) - \Ex_{z_i}\big[F(z_{1:n})\big]\Big\}\\
& = \Ex_{(i+1):n} \Big\{F(z_1, \dots, z_{i-1}, z_i, z_{i+1}, \dots, z_n) - \Ex_{z_i'}\big[F(z_1, \dots, z_{i-1}, z_i', z_{i+1}, \dots, z_n)\big]\Big\}\\
& = \Ex_{(i+1):n} \Ex_{z_i'} \Big\{F(z_1, \dots, z_{i-1}, z_i, z_{i+1}, \dots, z_n) - F(z_1, \dots, z_{i-1}, z_i', z_{i+1}, \dots, z_n)\Big\}
\end{aligned}
\end{equation}
 where, $z_i'$ is an \IID\ copy of $z_i$. We notice that 
\begin{equation}
\label{eq:fluctuation-bound}
    \begin{aligned}
& ~~ F(z_1, \dots, z_{i-1}, z_i, z_{i+1}, \dots, z_n) - F(z_1, \dots, z_{i-1}, z_i', z_{i+1}, \dots, z_n) \\
& = \textstyle\sup_{f \in \cF} \Big\{ \Ex\big[g_f(z_1)\big] -\frac1 {n} \sum_{i = 1}^{n} g_f(z_i) \Big\}\\
& ~~~~~~~~- \sup_{f \in \cF} \Big\{ \Ex\big[g_f(z_1)\big] -\frac1 {n} \sum_{i = 1}^{n} g_f(z_i) +\frac 1n g_f(z_i') - \frac1n g_f(z_i) \Big\}\\
& \le \textstyle\sup_{f \in \cF} \Big\{ - \frac 1n g_f(z_i') + \frac1n g_f(z_i) \Big\}
\end{aligned}
\end{equation}
 where the last inequality is obtained by setting $A_f = \Ex\big[g_f(z_1)\big] -\frac1 {n} \sum_{i = 1}^{n} g_f(z_i)$ and $B_f = \frac 1n g_f(z_i') - \frac1n g_f(z_i)$ in the following stream of inequalities
\[
\begin{aligned}
 \sup_{f\in \cF} \{A_f\} - \sup_{f\in \cF} \{A_f+ B_f\} 
& = \textstyle\sup_{f\in \cF} \{A_f+B_f - B_f\} - \sup_{f\in \cF} \{A_f+ B_f\} \\
& \le \textstyle\sup_{f\in \cF} \{A_f+B_f\} + \sup_{f\in \cF} \{-B_f\} - \sup_{f\in \cF} \{A_f+ B_f\}\\
& = \sup_{f\in \cF} \{-B_f\}\,,
\end{aligned}
\]
and that 
\begin{equation}
\label{eq:sup-diff-bound}
   \begin{aligned}
 \textstyle\sup_{f \in \cF} \Big\{ - \frac 1n g_f(z_i') + \frac1n g_f(z_i) \Big\}
& \textstyle\le \frac 1n \Big\{ \sup_{f \in \cF} |g_f(z'_i)| + \sup_{f \in \cF} |g_f(z_i)|\Big\}\\
& = \textstyle\frac 1n \Big\{ w^\star(z_i')\sup_{f \in \cF} |\ell(z_i')| + w^\star(z_i)\sup_{f \in \cF} |\ell(z_i')|\Big\}\\
& \le \textstyle\frac{\|\ell\|_\infty}{n}\big\{ w^\star(z_i') + w^\star(z_i)\big\}\,,
\end{aligned} 
\end{equation} where $\ell(z) = \ell(f(x), y)$. 
We use inequalities \eqref{eq:fluctuation-bound} and \eqref{eq:sup-diff-bound} in \eqref{eq:martingale-diff} and get the following
\begin{equation}
\label{eq:martingale-diff-ub}
\begin{aligned}
 \Ex_{(i+1):n}\big[F(z_{1:n})\big] - \Ex_{{i:n}}\big[F(z_{1:n})\big] &\le \frac{\|\ell\|_\infty}{n}\big\{ \Ex_{z_i'} [w^\star(z_i')] + w^\star(z_i)\big\}\,.
\end{aligned}
   \end{equation}
Now, we use \eqref{eq:martingale-sum} and \eqref{eq:martingale-diff-ub} to bound the moment generating function of $F(z_{1:n}) - \Ex[F(z_{1:n})]$ as seen in the following inequalities. For $\lambda > 0$
\begin{equation}
\label{eq:mgf-bound}
\begin{aligned}
  & ~~   \Ex\Big\{ \exp\Big(\lambda\big\{F(z_{1:n}) - \Ex[F(z_{1:n})]\big\}\Big)\Big\}\\
 & = \Ex\Big\{ \exp\Big(\lambda\sum_{i = 1}^n\big\{\Ex_{(i+1):n}\big[F(z_{1:n})\big] - \Ex_{{i:n}}\big[F(z_{1:n})\big]\big\}\Big)\Big\}\\
 & \le \Ex\Big\{ \exp\Big(\lambda\sum_{i = 1}^n\frac{\|\ell\|_\infty}{n}\big\{ \Ex [w^\star(z_i')] + w^\star(z_i)\big\}\Big)\Big\}\\
 & \le \Ex\Big\{ \exp\Big(\lambda\sum_{i = 1}^n\frac{\|\ell\|_\infty}{n}\big\{ w^\star(z_i')+ w^\star(z_i)\big\}\Big)\Big\}, ~~ \text{since} ~~ e^{\Ex\{X\}} \le \Ex\big\{e^X\big\}\\
 & =  \prod_{i = 1}^n\Ex\Big\{ \exp\Big(\frac{\lambda \|\ell\|_\infty}{n} w^\star(z_i)\Big)\Big\} \Ex\Big\{ \exp\Big(\frac{\lambda \|\ell\|_\infty}{n} w^\star(z_i')\Big)\Big\}\\
 & \le  \prod_{i = 1}^n \exp\Big(\frac{2c\lambda^2 \|\ell\|_\infty^2}{n^2} \Big) = \exp\Big( \frac{2c\lambda^2 \|\ell\|_\infty^2}{n} \Big)\,.
\end{aligned}
\end{equation} 
Following the bound on moment generating function in \eqref{eq:mgf-bound} we get 
\[
\begin{aligned}
  & ~~\Pr\big\{F(z_{1:n}) - \Ex[F(z_{1:n})] > t\big\}\\
  &\le e^{-\lambda t} \Ex\Big\{ \exp\Big(\lambda\big\{F(z_{1:n}) - \Ex[F(z_{1:n})]\big\}\Big)\Big\}\\
  & = \exp\Big(-\lambda t + \frac{2c\lambda^2 \|\ell\|_\infty^2}{n} \Big)\,,
\end{aligned}
\] where letting $\lambda = nt /(4c \|\ell\|_\infty^2)$ we obtain
\[
\Pr\big\{F(z_{1:n}) - \Ex[F(z_{1:n})] > t\big\} \le \exp \Big( -\frac{nt^2}{8c \|\ell\|_\infty^2} \Big)\,,
\] and letting $t = \|\ell\|_\infty\sqrt{8c\log(1/\delta)/n}$ we establish the lemma with $C = \|\ell\|_\infty \sqrt{8c}$.

\end{proof}

\section{Experiment details}

\subsection{Data details}
\label{sup:exp:data}

\subsubsection{\waterbirds}

\paragraph{Data} The training data has 4795 sample points with group-wise sample sizes $\{0: 3498,  1: 184,   2: 56, 3: 1057\}$. We combine the test and the validation data to create test data which has 8192 sample points, and the group vise sample sizes are $\{0: 3189, 1:3187,  2:908,  3:908\}$. The images are embedded into 512 dimensional feature vectors using ResNet18 \citep{he2016Deep} pre-trainined on Imagenet \citep{deng2009imagenet}, which we use as covariates. Recent work suggests that using pre-trained features without additional fine-tuning of the feature extractor is beneficial for out-of-distribution generalization \citep{kumar2022FineTuning,rosenfeld2022domain}.

\paragraph{Source and target domains} For the source domain we use the original training set of images. We consider five different target domains: (1) the target domain with all the groups $g \in \{0, 1, 2, 3\}$ from test data, (2) with groups $g \in \{0, 3\}$ \ie, landbirds on land backgrounds and waterbirds on water backgrounds, (3) with groups $g \in\{ 0, 2\}$, \ie, landbirds on land backgrounds and waterbirds on land backgrounds, (4) with groups $g \in \{1, 3\}$, and (5) with groups $g \in \{1, 2\}$. Note that all of the target domains have landbirds and waterbirds.

\subsubsection{\breeds}

\paragraph{Data} \breeds\ \citep{santurkar2020breeds} is a subpopulation shift benchmark derived from ImageNet \citep{deng2009imagenet}. It uses the class hierarchy to define groups within classes. For example, in the Entity-30 task considered in this experiment, class fruit is represented by strawberry, pineapple, jackfruit, Granny Smith in the source and buckeye, corn, ear, acorn in the target. Each source and target datasets are split into training and test datasets.  In the source domain the training (resp. test)  data has 159037 (resp. 6200) sample points, whereas in the target domain the sample sizes are 148791 (resp. 5800) for training (resp. test) data. There are 30 different classes in both source and target domains. The highest (resp. lowest) class proportion in source training data is 4.9\% (resp. 1.58\%) and in target training data is 4.9\% (resp. 1.53\%). 
Here, the images are embedded using SwAV \citep{caron2020unsupervised}. The embedding is of dimension 2048, which we consider as covariates for our analysis. As in the \waterbirds\ experiment, we do not fine-tune the embedder.

\paragraph{Source and target domains}
In our \breeds\ case study we mix a small amount of labeled target samples into the source data. We mix $\pi$ proportion of labeled target samples  into the source domain  and  evaluate the performance of our method for several mixing proportions ($\pi$). Below we describe the step by step procedure for creating mixed source and target datasets:
\begin{enumerate}
\item In both the source and target domains we combine the training and test datasets.
    \item Let $m$ be the sample size of the combined source data. We add $\lfloor m\pi\rfloor$ many labeled target samples into the source data.
    \item Resulting source and target datasets are then split to create training ($80\%$) and test ($20\%$) data.
\end{enumerate}

\subsection{Model details}
\label{sup:exp:model}

\subsubsection{Implementation for ExTRA}
We describe the implementation details for ExTRA weights in Algorithm \ref{alg:extra}. 
\paragraph{The normalization regularizer $\lambda$} 
 is required to control the value of the normalizer $\hat N_t$. It makes sure that the value of $\hat N_t$ remains close to 1, as the function $x + x^{-1}, x>0$ is minimized at $x = 1$. The regularizer is particularly important when the feature distribution  between source and target data has very little overlap (happens in \breeds\ case study).

\begin{algorithm}
\caption{Exponential Tilt Reweighting Alignment (ExTRA)\label{alg:extra}}
\begin{algorithmic}[1]
\Require 
\begin{itemize}
\item \textbf{Dataset:} labeled source data $\{(X_{P, i}, Y_{P, i})\}_{i = 1}^{\nsource}$ and unlabeled target data $\{X_{Q, i}\}_{i = 1}^{n_Q}$. 
    \item \textbf{Hyperparameters:} learning rate $\eta>0$, batch size $B\in \bbN$,   normalization regularizer $\lambda > 0$. 
    \item \textbf{Probabilistic source classifier:} $\heta_P : \cX \to \Delta^K$.
    \item \textbf{Initial values}: $\{(\hat \theta_{k, 0}, \hat \beta_{k, 0})\}_{k = 1}^K$
\end{itemize}

\State Initialize $\hat \theta_0$ at some value. \Repeat{ $t \ge 0$}
\State sample minibatchs
 $(X_{P,  1}, Y_{P,  1}), \dots, (X_{P,  B}, Y_{P,  B}) \sim \hP$, and 
     $X_{Q, 1}, \dots, X_{Q, B}\sim \hQ_X$
\State Compute loss 
$
\hat L_t = \frac 1 B \sum_{i =1}^B \log\Big\{\sum_{k=1}^K\heta_{P,k}(X_{Q, i})\exp(\htheta_{k, t}^\top T(X_{Q, i}) + \hat \beta_k\Big\} 
$ and normalizer $\hat N_t = \frac 1B\sum_{i = 1}^B \exp\big\{ \htheta_{Y_{P, i}, t}^\top T(X_{P, i}) + \hat \beta_{Y_{P, i}, t} \big\}$
\State Objective $\hat O_t =  - \hat L_t + \log(\hat N_t) + \lambda \hat N_t + \lambda \hat N_t^{-1} $
\State Update $\hat \theta_{k, t+1} \gets \hat \theta_{k, t} - \eta \partial_{\theta_k}\hat O_t\{(\hat \theta_{k, t}, \hat \beta_{k, t}), k = 1, \dots , K\} $ and $\hat \beta_{k, t+1} \gets \hat \beta_{k, t} - \eta \partial_{\beta_k}\hat O_t\{(\hat \theta_{k, t}, \hat \beta_{k, t}), k = 1, \dots , K\} $
\Until{converges}
\State Estimated value $\{(\hat \theta_{k}, \hat \beta_{k})\}_{k = 1}^K$
\State $\hat \alpha _k \gets \hat \beta_k - \log \hat N(\{(\hat \theta_{k}, \hat \beta_{k})\}_{k = 1}^K)$\\

\Return parameters $\{(\hat \theta_{k}, \hat \alpha_{k})\}_{k = 1}^K$ and the weight function $\weight(x, y) = \exp(\hat \theta_y T(x) + \hat \alpha_y)$
\end{algorithmic}
\end{algorithm}

\subsubsection{\waterbirds}

\paragraph{Source classifier $\heta_P$} is a logistic regression model fitted on source data using \texttt{sklearn.linear\_model.LogisticRegression} module with the parameters \{solver = `lbfgs', C = 0.1, tol = 1e-6, max\_iter=500\} and the rest set at their default values. We also use several calibration techniques for the source classifier \citep{shrikumar2019calibration}: (1) no model calibration (none), (2) temperature scaling (TS), (3) bias corrected temperature scaling (BCTS), and (4)  vector scaling (VS). For TS, BCTS and VS  we use the implementation in \citet{shrikumar2019calibration}. 

\paragraph{ExTRA importance weights}
In each iterations of ExTRA importance weight calculations (Algorithm \ref{alg:extra}) for  \waterbirds\ data we fix an initialization of the parameters and compute the parameters for several values of the hyperparameters learning rate $\eta\in \{5\times 10 ^{-4}, 4 \times 10^{-5}\}$, batch size $B = 500 $, epochs $E \in \{100, 200, 400\}$ and source model calibrations \{none, TS, BCTS, VS\}. The details can be found in the supplementary codes. Since there are significant overlap between source and target feature distributions we set $\lambda = 0$. We select the hyperparameter setup that produces the lowest value of the objective $-\hat L + \log (\hat N) $  over the full data $\{(X_{P, i}, Y_{P, i})\}_{i = 1}^{\nsource}$ and  $\{X_{Q, i}\}_{i = 1}^{n_Q}$.

\paragraph{Runtime} ExTRA algorithm requires solving a simple stochastic optimization problem. In the \waterbirds\  case-study, 100 epochs of the Adam \citep{kingma2017Adam} optimizer took $11.92 \pm 0.48$ seconds.

\paragraph{Model selection} We considered 120 different models which are logistic regression models of three categories: 
\begin{enumerate}
    \item A vanilla  model.
    \item A model fitted on weighted data to balance the class proportion on source data.
    \item A model fitted on weighted data to balance the group proportions on the source data. 
\end{enumerate}  Each of these models are fitted with scikit-learn logistic regression module where we use 2 different regularizers, $\ell_1$ and $\ell_2$ and 20 different regularization strengths (\texttt{numpy.logspace(-4, -1, 20)}) and we use \texttt{liblinear} solver to fit the models. Rest of the parameters are set at their default values. 

For model selection experiments the source and target data are each split into equal parts to create source and target training and test datasets. The models are fitted using test data on the source domain. We then calculate its (1) ScrVal accuracy using the source training data, (2) ACT-NE accuracy using labeled training data on source and unlabeled training data on target, (3) ExTRA accuracy on training data on source, and finally (4) oracle target accuracy on test data on target. We then summarize the accuracies in (1) oracle \emph{target accuracies} for the models chosen according to the best ScrVal, ACT-NE and ExTRA accuracies on the source domain, and (2) the  \emph{rank correlations} between the ScrVal, ACT-NE, and ExTRA accuracies with the corresponding oracle target accuracies. 
 
\subsubsection{\breeds}

\paragraph{Source classifier $\heta_P$} The source classifier $\heta_P$ used in \breeds\ case study is similar to the one in \waterbirds. It uses same model and parameters to obtain a probabilistic classifier $\heta_P$. We use bias corrected temperature scaling (BCTS) \citep{shrikumar2019calibration} for calibrating $\heta_P$. 

\paragraph{ExTRA importance weights} We set the hyperparameters at some fixed values and obtain our weights for several random initializations  of the tilt parameters. The hyperparameter values are: (1) learning rate $\eta = 10^{-4}$, (2) batch size $B = 1500$, (3) number of epochs $\text{epochs} = 500$, and (4) regularization strength for normalizer $\lambda = 10^{-6}$. Rest of the setup is same as in \waterbirds. 

\subsection{Additional results}
\label{sup:exp:results}

\subsubsection{\waterbirds}
\label{sup:exp:waterbirds}

\paragraph{Precision and recall} We report the precision and recall for the weights. They are defined as following: within an $x$ proportion of samples with the highest weights (call this set $A$)

\begin{enumerate}
    \item precision is the proportion of sample points from the groups comprising the target domain in $A$, \ie\  \[\frac{\#\{\text{sample points in }A \text{ with } g \in \{\text{target groups}\}\}}{|A|}\,, \] and,
    \item recall is the ratio between the number of sample points in $A$ that are from the target groups and the total number of points in source data that are from the target groups, \ie \[\text{precision} = \frac{\#\{\text{sample points in }A \text{ with } g \in \{\text{target groups}\}\}}{\#\{\text{sample points in source data}\text{ with } g \in \{\text{target groups}\}\}}\,. \] 
\end{enumerate}
Target domains consisting of a majority and a minority group, \eg\ $\{0,2\}$, are noticeably imbalanced in the source data, thus we also report precision and recall conditioned on the class label (\ie\ treating source and target as consisting of either only landbirds ($y=0$) or only waterbirds ($y=1$) in the aforementioned precision and recall definitions).

We report results for four target domains in Figure \ref{fig:waterbirds-precision-recall-all}. Overall, the ExTRA weights are informative (precision curves have downward trends and the recall curves are above the non-informative baseline, \ie\ solid black lines). We note that for class $y = 0$  for target domain $\{0, 2\}$ and for class $y = 1$ for target domain $\{1, 3\}$ we see that the ExTRA weights are almost non-informative (precision curve is almost flat and the 
recall curve is almost aligned with the baseline).  This is due to the group imbalance within a class. In the example of class $y = 0$ with target domain $\{0, 2\}$,  the groups with $y = 0$ are $g = 0$ and $g = 1$. Since the sample size for $g = 1$ is very small compared to $g = 0$, most of the samples in $y = 0$ class are from the correct group $g =0$ when we consider $\{0, 2\}$ as our target domain, and any weights would have precision-recall curves that are close to the non-informative baseline. Similar behavior is observed for the other example. 
\begin{figure}
    \centering
    \includegraphics[scale = 0.35]{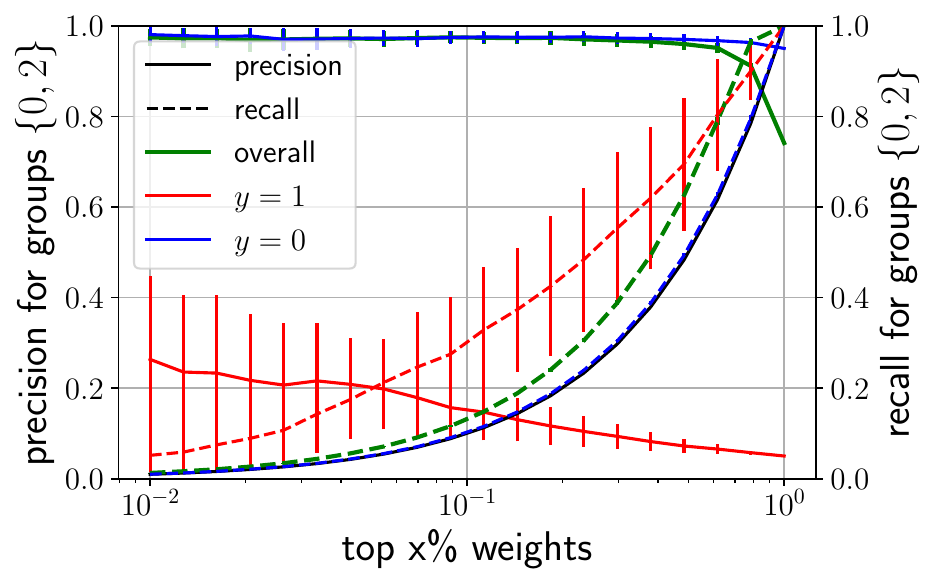} ~ \includegraphics[scale = 0.35]{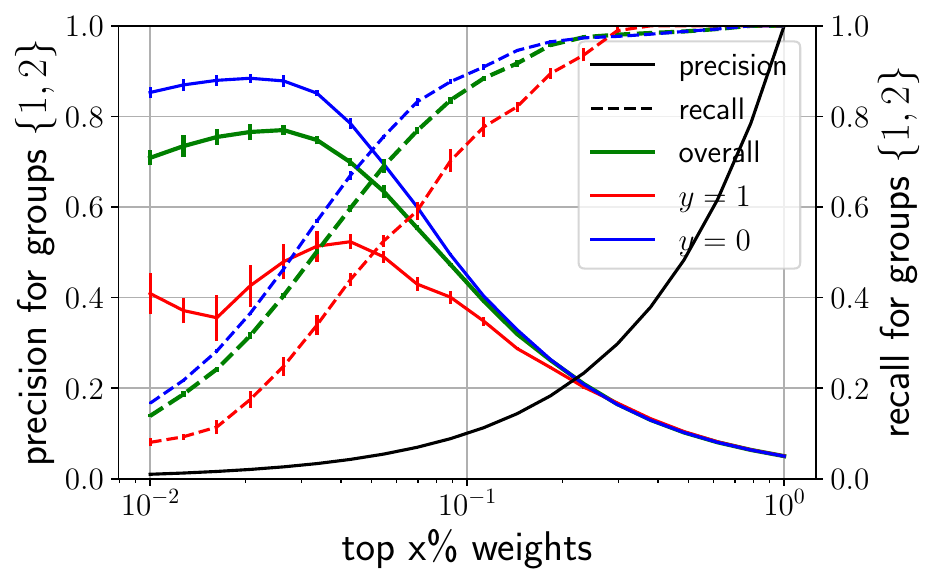}\\
    \includegraphics[scale =
    0.35]{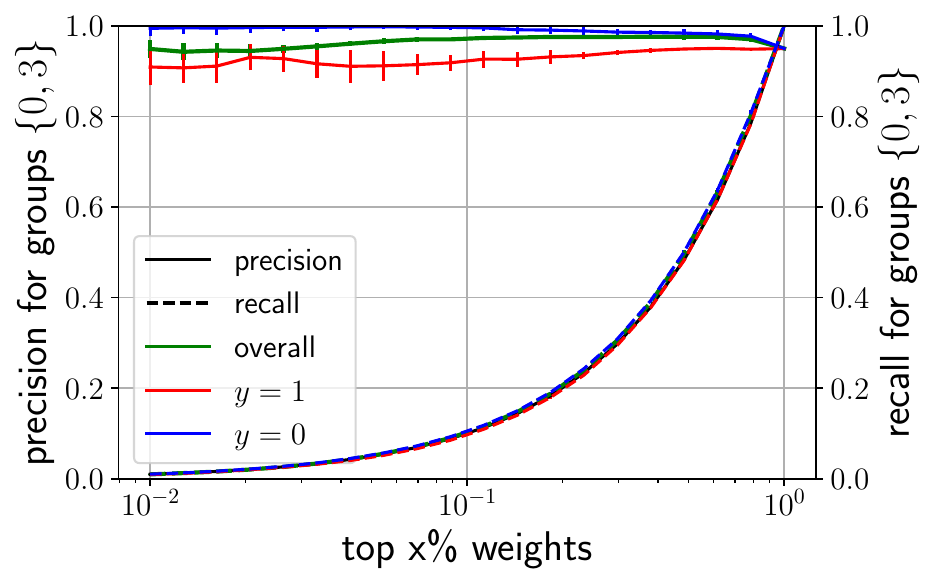} ~ \includegraphics[scale = 0.35]{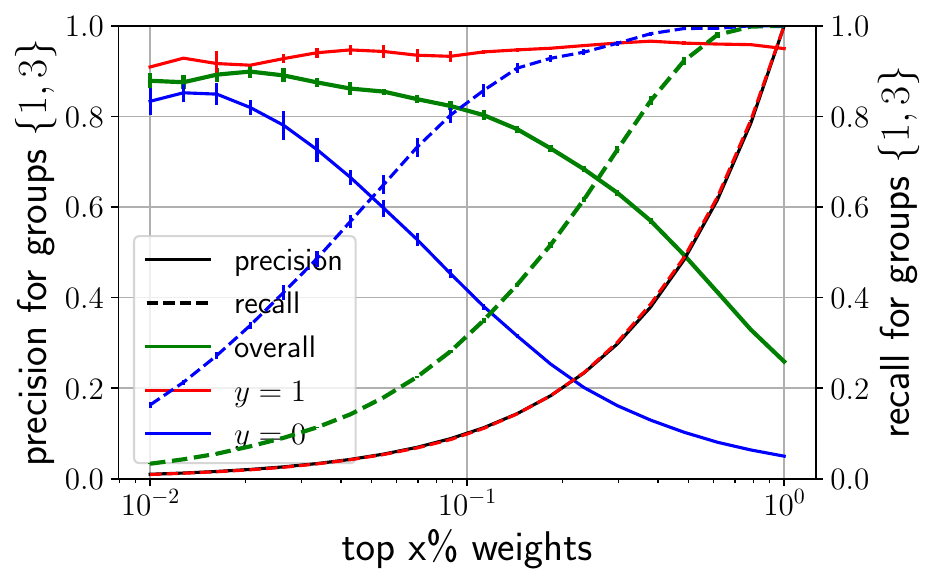}
    \caption{ExTRA precision and recall on \waterbirds\ for different targets. The black solid line refers to a baseline for the recall curve when the weights are completely non-informative of the target domain.}
    \label{fig:waterbirds-precision-recall-all}
\end{figure}

\paragraph{Upweighted images} We visualize images from the \waterbirds\ dataset corresponding to the 16 largest ExTRA weights for the $\{1,2\}$ target domain consisting of the two minority groups. Among these 16 images, 12 correspond to the correct groups, \ie\ either landbirds on water or waterbirds on land. We emphasize that in the source domain there are only 5\% of images corresponding to groups $\{1,2\}$ and ExTRA weights upweigh them as desired. The 4 images from the other groups (highlighted with red border) are: (i) 2nd row, 3rd column (waterbird on water); (ii) 3rd row, 2nd column (waterbird on water); (iii) 4th row, 3rd column (waterbird on water); (iv) 4th row, 4th column (landbird on land). Arguably, the background in (i) is easy to confuse with the land background and the blue sky in (iv) is easy to confuse with the water background, suggesting that these images might be representative of the target domain of interest despite belonging to different groups.

\begin{figure}
    \centering
    \includegraphics[scale=0.3]{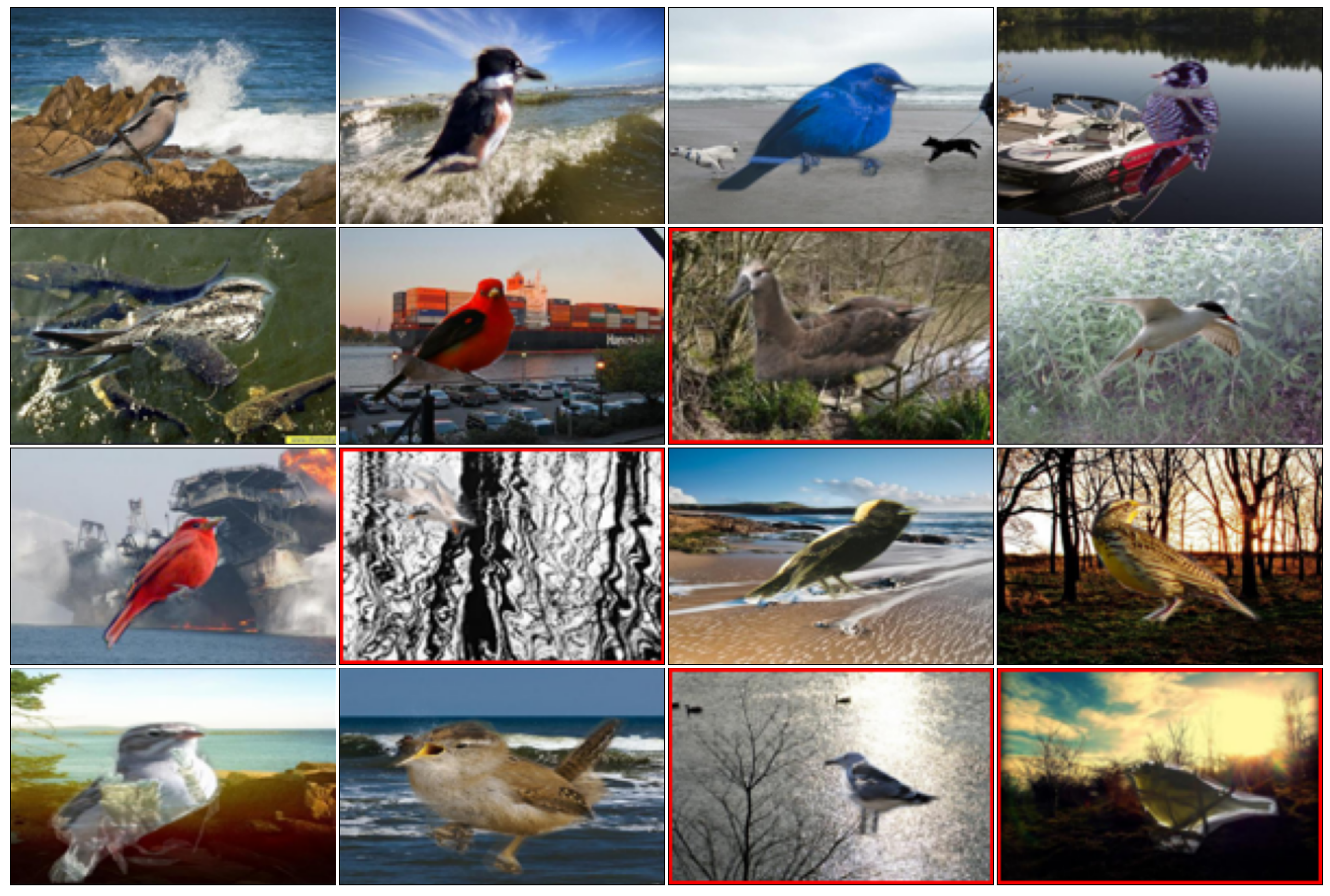}
    \caption{\waterbirds\ images with 16 largest ExTRA weights for $\{1,2\}$ as the target domain. The four images highlighted with red border are from other groups.}
    \label{fig:waterbirds-img}
\end{figure}

\paragraph{Area under the Receiver Operating Characteristic curve} Here we compare the area under the curve for target receiver operating characteristic curves (we call it target AUC-ROC). The behaviors for target AUC-ROC's are similar to target accuracies (Figure \ref{fig:waterbirds}). 
\begin{figure}
    \centering
    \includegraphics[scale = 0.5]{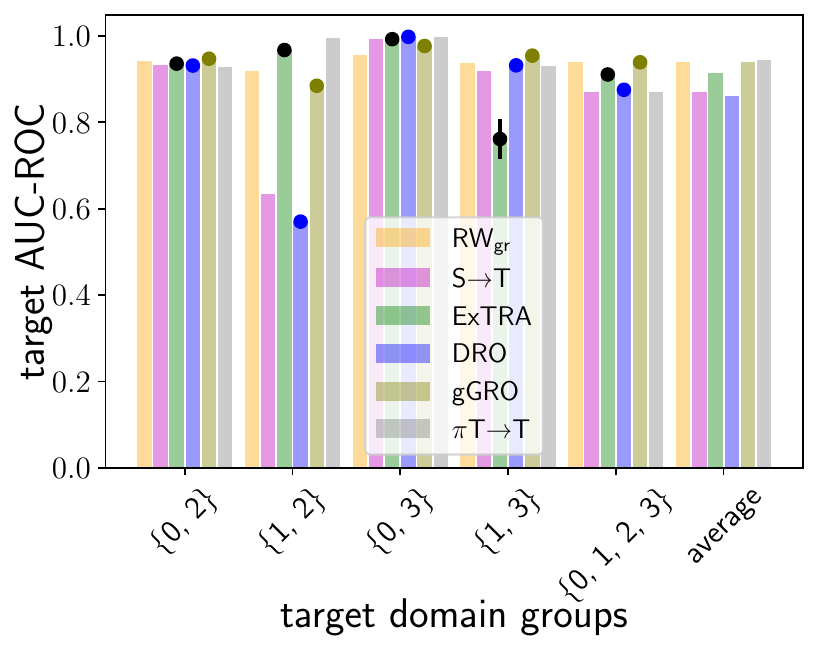}
    \caption{Area under the curve for target receiver operating characteristic curves  for \waterbirds.}
    \label{fig:waterbird-auc}
\end{figure}

\subsubsection{\breeds}
We present precision and recall curves for the target samples identified within the source samples with larger ExTRA weights (analogous to the corresponding \waterbirds\ experiment) in Figure \ref{fig:breeds-precision-recall} for varying mixing proportion $\pi$. In comparison to \waterbirds, we note that both precision and recall are lower, which we think is due to a larger amount of the original source samples representative of the target domain distribution as can be seen from the improved performance of the ExTRA fine-tuned model in Figure \ref{fig:breeds} even when $\pi=0$.

\begin{figure}
    \centering
    \includegraphics[scale = 0.5]{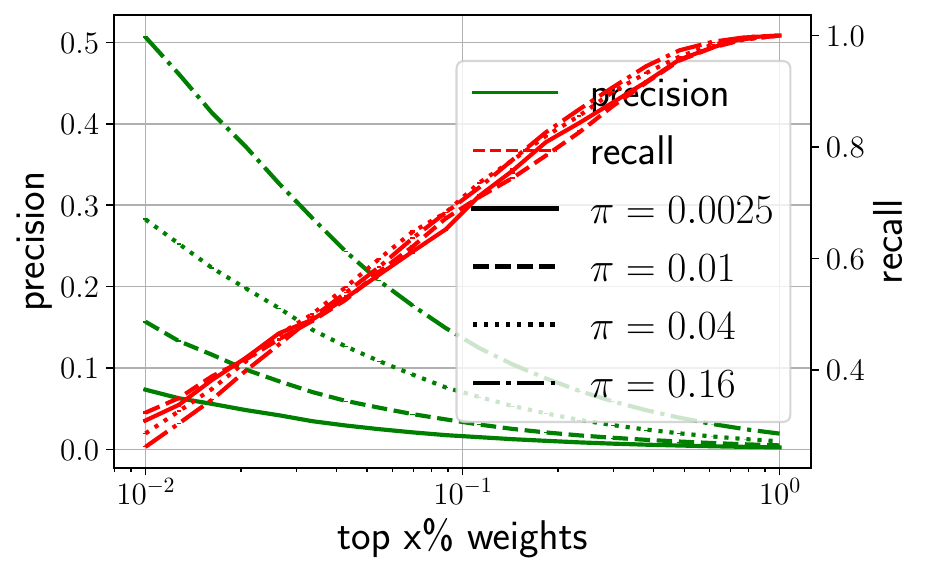}
    \caption{ExTRA precision and recall on \breeds.}
    \label{fig:breeds-precision-recall}
\end{figure}
\section{Synthetic experiment: normal mixture}
\label{sec:normal-mixture}

\begin{wrapfigure}[14]{r}{0.5\linewidth}
\vspace{-0.25in}
    \centering
    \includegraphics[width=\linewidth]{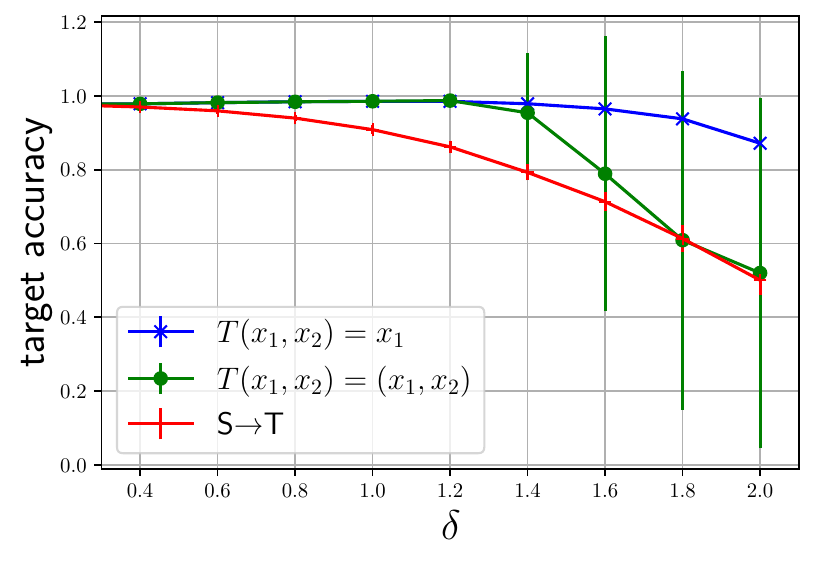}
    \vspace{-0.23in}
    \caption{Target accuracies for different models and for different $\delta$ values. }
    \label{fig:LDA-exptilt}
\end{wrapfigure}

We demonstrate utility of the ExTRA weights for reweighing source data in the following synthetic experiment. As seen in Figure \ref{fig:LDA-data} both the source and target distributions are associated with binary classification tasks and their class conditionals are normally distributed in $\reals^2$. More specifically, the class conditionals in source distribution have means $(-\delta, -2)^\top$ and $(\delta, 2)^\top$ for classes 1 and 0 resp., whereas the means in target distributions are $(\delta, -2)^\top$ and $(-\delta, 2)$. All the class conditionals have variances $\bI_2$. This is an instance of concept drift which does not fall under subpopulation shift (considering classes as subpopulations).

\begin{figure}
    \centering
    \includegraphics[width=0.45\linewidth]{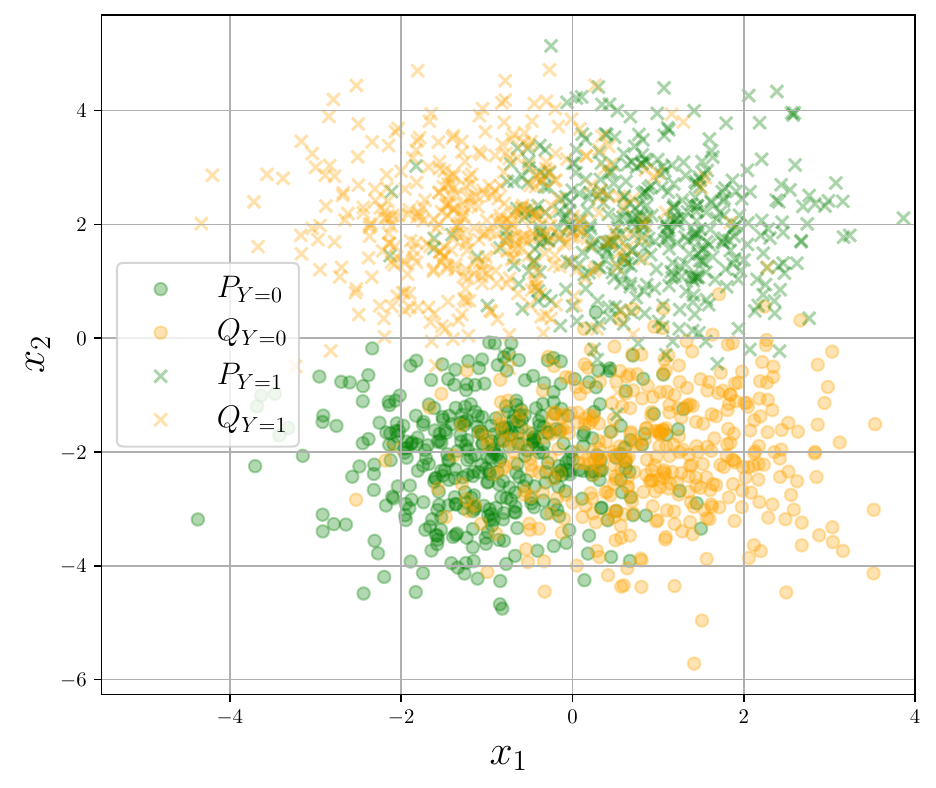}~\includegraphics[width=0.45\linewidth]{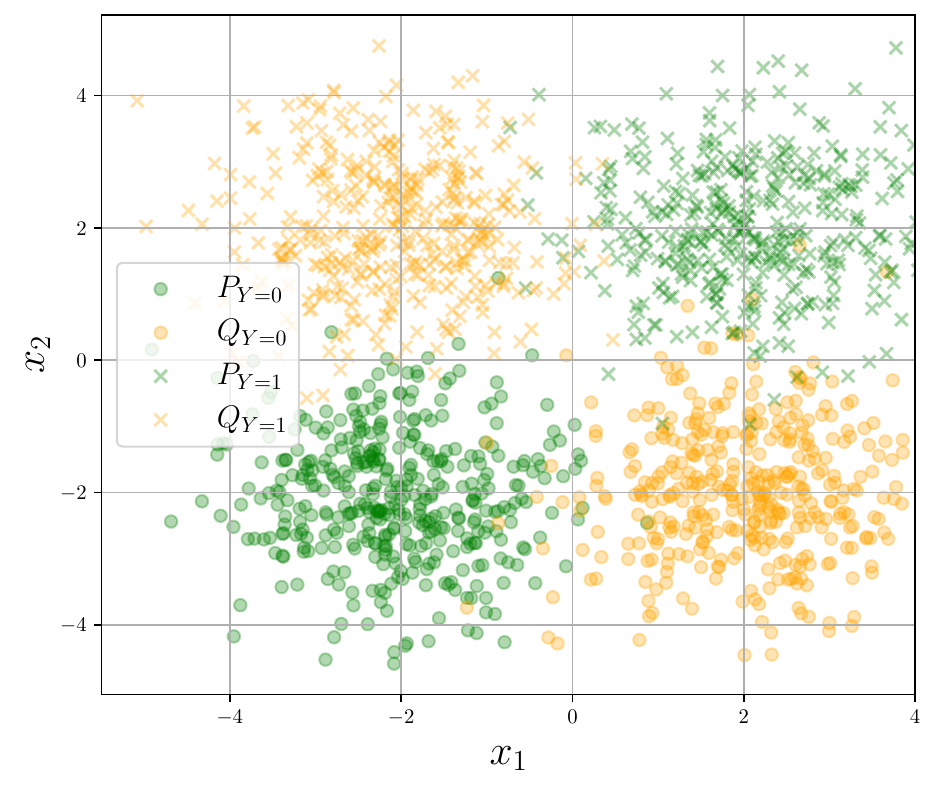}
    \caption{Source and target data with $\delta = 1$ (left) and $\delta = 2$ (right).}
    \label{fig:LDA-data}
\end{figure}

The distribution shift in this example satisfies exponential tilt assumption with sufficient statistic $T(x_1, x_2) = x_1$. In our experiment we use two sufficient statistics: (1) $T(x_1, x_2) = x_1$ - an ideal sufficient statistic for this example and (2) $T(x_1,x_2) = (x_1, x_ 2) $ - a simple choice for a sufficient statistic. 
Here, we vary $\delta$ to control the overlap between the class conditionals in corresponding classes. For a small $\delta$ the source and target class conditionals for class 0 (or 1) have overlapping support, satisfying the anchor-set assumption in Definition \ref{def:anchor-set}.

We compare the target accuracies of three different models trained on: (1) source data (S->T), (2) ExTRA weighted source data using $T(x_1, x_2) = x_1$, and (3) ExTRA weighted source data using $T(x_1, x_2) = (x_1, x_2)$. As we see in Figure \ref{fig:LDA-exptilt}, the ideal sufficient statistic $T(x_1, x_2) = x_1$ produces better target accuracy than the other two models. We also observe that $T(x_1, x_2) = (x_1, x_2)$ has better target accuracy than S->T for small values of $\delta$. The reason for such observation goes back to the anchor set condition in Definition \ref{def:anchor-set}. As we see in Figure \ref{fig:LDA-data}, left plot, a small $\delta$ ($\delta = 1$) results in support overlap between source and target class conditionals for class $0$ (and $1$), and the overlapped region of support works as anchor set. On contrary, large $\delta$'s (Figure \ref{fig:LDA-data}, right plot) has very little overlap, resulting in violation in  the anchor set assumption and leading to a non-identifiability in the exponential tilt model for sufficient statistic $T(x_1, x_2) = (x_1, x_2)$. Nevertheless, the ideal sufficient statistic $T(x_1, x_2) = x_1$ continues to work even for large $\delta$.

\begin{figure}
    \centering
    \includegraphics[width=0.45\linewidth]{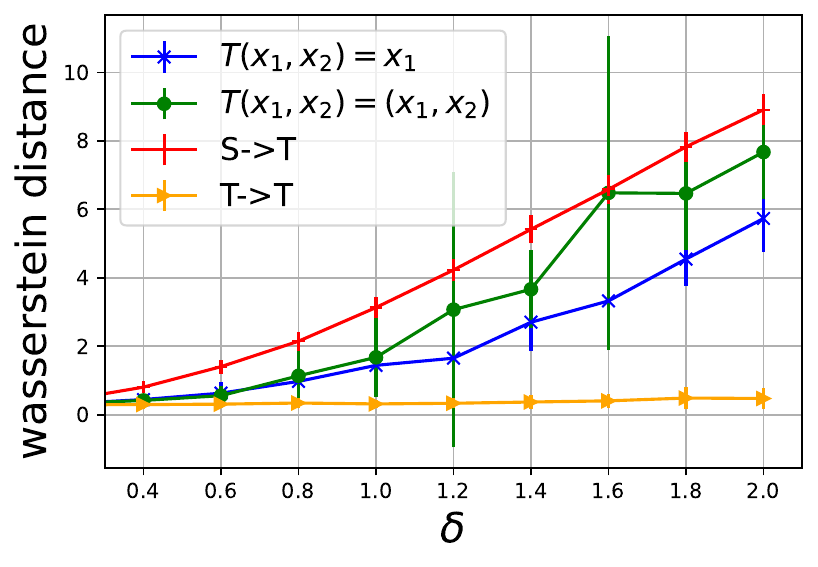}~\includegraphics[width=0.45\linewidth]{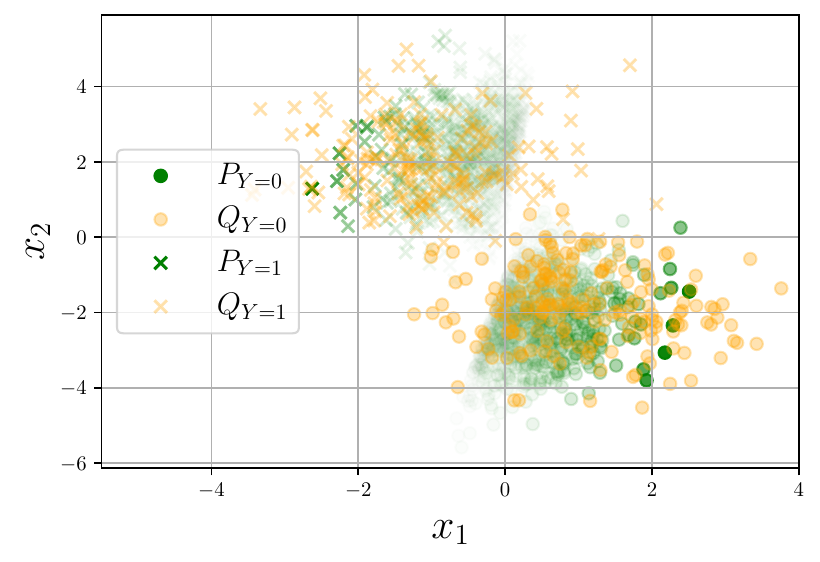}
    \caption{\emph{Left}: Wasserstein distances between source vs target, weighted source vs target and target vs target data. \emph{Right}: Scatter plot for weighted source (highlighted according to the ExTRA weights) and target data. }
    \label{fig:LDA-distribution-matching}
\end{figure}


We also investigate whether reweighed source data with ExTRA weights better approximates the target data compared to the uniformly weighted source data. We first compare the following Wasserstein distances between target and: (1) source (S->T), (2) target (T->T), (3) ExTRA weighted source with $T(x_1, x_2) = (x_1, x_2)$, and (4) ExTRA weighted source with $T(x_1, x_2) = x_1$ for several values of $\delta$. 
As we see in Figure \ref{fig:LDA-distribution-matching} left plot, ExTRA weighted source data is closer to the target data compared to the uniformly weighted source data. Similar to the target accuracy inspection (Figure \ref{fig:LDA-exptilt}), the sufficient statistic  $T(x_1, x_2) = x_1$ shows better performance in adapting to the target distribution than $T(x_1, x_2) = (x_1, x_2)$. In  Figure \ref{fig:LDA-distribution-matching} right panel, we compare the weighted source data for $\delta = 1$ when the ExTRA weights are calculated with sufficient statistic $T(x_1, x_2) = (x_1, x_2)$. Compared to the uniformly weighted source data (see Figure \ref{fig:LDA-data}, left plot) we see that ExTRA weighted source data is also closer to the target one qualitatively.

\end{document}